%% file: main_ext.tex
\documentclass[sigconf, nonacm]{acmart}

\usepackage{pvldb}

\renewcommand\vldbdoi{XX.XX/XXX.XX}
\renewcommand\vldbpages{XXX-XXX}
\renewcommand\vldbavailabilityurl{https://github.com/bigbases/CoPE-CGL} 

\usepackage[utf8]{inputenc}
\usepackage[T1]{fontenc}

\usepackage{hyperref}
\usepackage{url}
\usepackage{booktabs}
\usepackage{amsfonts}
\usepackage{nicefrac}
\usepackage{microtype}
\usepackage{xcolor}
\usepackage{color}
\usepackage{kotex}
\usepackage{multirow}
\usepackage{textcomp}
\usepackage{comment}
\usepackage{caption}
\usepackage{subcaption}
\usepackage{algorithm}
\usepackage{algpseudocode}
\usepackage{makecell}
\usepackage{xspace}
\usepackage{enumitem}
\usepackage{amsmath}
\usepackage{amsthm}
\usepackage{algorithmicx}
\usepackage[normalem]{ulem}
\usepackage{graphicx}
\usepackage{wrapfig}
\usepackage{balance}

\newtheorem{proposition}{Proposition}
\newtheorem{corollary}{Corollary}

\newcommand{\model}{CoPE-CGL\xspace}

\newif\ifrevmarks \revmarksfalse
\ifrevmarks
  
  \DeclareRobustCommand{\pd}[1]{\textcolor[rgb]{0.5,0.0,0.5}{\relax#1}}
\else
  
  \DeclareRobustCommand{\pd}[1]{#1}
\fi

\begin{document}
\title{CoPE-CGL: Community-aware Prototype Extraction for Replay-Free Continual Graph Learning (Extended Version)}

\author{Bo-Young Lim}
\orcid{0000-0002-3325-7545}
\affiliation{%
  \department{Graduate School of Data Science}
  \institution{Seoul National University of Science and Technology}
  \city{Seoul}
  \country{South Korea}}
\email{esob618@seoultech.ac.kr}

\author{Seokhwan Kim}
\orcid{0009-0003-4865-6981}
\affiliation{%
  \department{Graduate School of Data Science}
  \institution{Seoul National University of Science and Technology}
  \city{Seoul}
  \country{South Korea}}
\email{kimauve@seoultech.ac.kr}

\author{Kisung Lee}
\orcid{0000-0003-4367-4374}
\affiliation{%
 \department{Division of Computer Science and Engineering}
 \institution{Louisiana State University}
 \city{Baton Rouge}
 \state{Louisiana}
 \country{USA}}
 \email{klee76@lsu.edu}

\author{Hyuk-Yoon Kwon}
\authornote{Corresponding author}
\orcid{0000-0002-1125-6533}
\affiliation{%
  \department{Graduate School of Data Science}
  \institution{Seoul National University of Science and Technology}
  \city{Seoul}
  \country{South Korea}}
\email{hyukyoon.kwon@seoultech.ac.kr}

\begin{abstract}
Graphs managed by data systems evolve: new subgraphs and node classes keep arriving, and a model learned from them must be maintained without retraining from scratch or retaining all past graph data. Continual graph learning (CGL) addresses this problem while mitigating catastrophic forgetting, but existing methods are typically tailored to a single evaluation setting, either transductive (disjoint node classes within a shared graph) or inductive (task subgraphs that share nodes, edges, and classes with earlier tasks). They also commonly store raw past subgraphs in replay buffers or require task identifiers at inference time. We propose \textbf{Co}mmunity-aware \textbf{P}rototype \textbf{E}xtraction for \textbf{C}ontinual \textbf{G}raph \textbf{L}earning (simply, \textbf{\model}), a replay-free CGL framework in which a parameter-free structural prior, precomputed by multi-hop Laplacian smoothing, anchors learning, memory, and inference. \model has three components: (i)~\textit{topology-infused bottleneck adaptation} (TIBA), a lightweight adapter trained on the frozen prior; (ii)~\textit{community-aware prototype extraction}, which distills each task into a few non-learned community-level prototypes; and (iii)~\textit{graph-level dynamic routing}, which infers the task of a test graph by a majority vote of its nodes against the prototype memory. As with prior task-free methods, the default protocol assumes that each test graph comes from a single learned task; a node-level variant relaxes this assumption. Each task adds at most a constant amount of storage (an adapter and a few prototypes), and our model keeps no copy of past graph data. At equal storage and without task identifiers, \model matches the average performance (AP) of the strongest replay baseline even when that baseline is given the true task of each test graph (AP of 95.17\% and 95.72\% for \model\ vs.\ 94.28\% and 95.82\% for the baseline on Arxiv and CoraFull), while running $5.9$--$14.0\times$ faster with zero forgetting. On four transductive and three inductive benchmarks and on mixed streams that interleave the two, \model outperforms the second-best of up to eleven baselines by up to 1.41, 6.32, and 9.33 percentage points of AP, respectively, runs 1.3--7.9$\times$ faster end-to-end than the most accurate competitor on each benchmark, and scales to the largest benchmark graph. On test graphs that mix two tasks, its node-level variant gives the best AP among CGL methods. The design targets homophilic graphs with community structure.
\end{abstract}

\maketitle

\noindent\textbf{Artifact availability.} The source code, data, and other artifacts are available at \url{https://github.com/bigbases/CoPE-CGL}.


\input{section/ch1_introduction}
\input{section/ch2_problem_definition}

\input{section/ch3_proposed_method}

\input{section/ch4_experiments}

\input{section/ch5_related_works}

\input{section/ch6_conclusion}


\bibliographystyle{ACM-Reference-Format}
\bibliography{main_ext}

\appendix

\section{Dataset and Implementation Details}
\label{appsec:datasets}

\subsection{Dataset Details}

This appendix expands on the benchmarks summarized in Section~4.1 by giving per-dataset feature and construction details. Per-dataset statistics are reported in Table~1.

\noindent\textbf{Transductive benchmarks (CGLB protocol~\cite{zhang2023cglb}).}
\textbf{CoraFull} and \textbf{Arxiv} are citation networks with 70 and 40 classes, respectively; CoraFull's node features are bag-of-words vectors of paper abstracts, and Arxiv's node features are averaged word embeddings from paper titles and abstracts. \textbf{Reddit} is a social-network graph in which nodes are posts grouped by subreddit community, edges represent shared commenters, and features are derived from post metadata. \textbf{Products} is an Amazon co-purchasing network in which nodes are items, edges are frequent co-purchases, and features are dimensionality-reduced bag-of-words representations from product descriptions.

\noindent\textbf{Inductive benchmarks (TACO protocol~\cite{han2024taco}).}
\textbf{Kindle} is an Amazon co-purchasing network spanning six item categories; \textbf{DBLP} and \textbf{ACM} are academic citation graphs covering four research domains each. Following the TACO protocol, the global graphs are partitioned into sequential task subgraphs by node timestamp, so edges point only to past entities, and subsets of classes are dynamically masked at each task to simulate evolving label spaces.

\noindent\textbf{Mixed-setting streams.}
Mixed-setting datasets are constructed on top of the seven benchmarks using the randomized scheduler described below (Randomized alternating scheduler): the transductive-to-inductive occurrence ratio is configurable, the task sequence of CoraFull, Arxiv, Reddit, and Products is randomly permuted for each independent run, whereas Kindle, DBLP, and ACM read fixed pre-built streams, and overlapping nodes for inductive tasks are uniformly sampled from the historical node pool at a predefined rate. The first task is always transductive to establish base knowledge.

\noindent\textbf{Randomized alternating scheduler.} To prevent overfitting to a deterministic task progression, the regime of each task is drawn from a randomized scheduler. The first task ($k{=}1$) consists entirely of novel nodes and is fixed to the transductive setting $\mathcal{T}_{trans}$; for $k>1$, its regime $R_k \in \{\mathcal{T}_{trans}, \mathcal{T}_{ind}\}$ is drawn from a Bernoulli distribution with $P(R_k{=}\mathcal{T}_{ind}) = \rho \in [0,1]$, and the resulting type sequence is permuted independently per run. This ensures the model must handle arbitrary orderings of transductive and inductive tasks without prior knowledge of the schedule.

\subsection{Implementation Details}

\noindent\textbf{Evaluation metrics.} Average Performance (AP$\uparrow$) measures overall accuracy after all tasks; Backward Transfer (BWT$\uparrow$) measures how much training on later tasks impacts earlier ones (higher is better):
\[
\text{AP} = \frac{1}{T} \sum_{j=1}^{T} a_{T,j}, \qquad
\text{BWT} = \frac{1}{T-1} \sum_{j=1}^{T-1}( a_{T,j} - a_{j,j}),
\]
where $T$ is the total number of tasks and $a_{i,j}$ is the model's score on the test set of task $j$ after training on task $i$ (the classification accuracy, in \%, on all seven benchmarks and for every method, including TACO, whose original paper reports F1).

\noindent\textbf{Seeds and run-to-run variation.} The seeds control parameter initialization, shuffling, mini-batching, and baseline-specific sampling, and every method sees the same task order for a given seed, while dataset splits, class order, and hyperparameters are fixed. A seed does not reproduce \model exactly, even with PyTorch deterministic algorithms: ten re-runs of one Arxiv seed on an RTX A4000 give AP between 95.13\% and 95.24\%.

\noindent\textbf{Replay budgets and codebases.} For replay-based baselines, the memory budget scales with dataset size: 100/200/300/400 for CoraFull/Arxiv/Reddit/Products in the transductive setting, and 100/200/300 for Kindle/DBLP/ACM in the inductive setting. EWC~\cite{Kirkpatrick2017ewc}, MAS~\cite{Aljundi2018mas}, TWP~\cite{liu2021twp}, ER-GNN~\cite{zhou2021ergnn}, SSM~\cite{zhang2022ssm}, CaT~\cite{liu2023cat}, Finetuning, and Joint are run with the CaT codebase (\url{https://github.com/superallen13/CaT-CGL.git}), which builds on the CGLB benchmark (\url{https://github.com/QueuQ/CGLB.git}). Joint uses the same two-layer SGC~\cite{wu2019sgc} and, after each task, trains it further for 500 epochs on the training nodes of all tasks seen so far, with one classifier over all classes seen so far. On Products each epoch takes one update per task graph, as in CaT's batch mode; on the other datasets it takes one update on the union of the task graphs. Its running time also includes extra routing and evaluation passes. TACO is run with its official code (\url{https://github.com/hanxiaoxue114/TACO.git}). TPP (\url{https://github.com/mala-lab/TPP.git}), TAAM (\url{https://github.com/1iuJT/TAAM_AAMAS2026}), CAGNN~\cite{xiao2026cagnn} (\url{https://github.com/WonderHeiYi/CAGNN.git}), and SAFER~\cite{zhang2026safer} (\url{https://github.com/ZhangWeichao0824/SAFER.git}) are ported from their official code into our pipeline, so that all methods share the same task streams and evaluation; TWP, CAGNN, and SAFER keep the backbones their methods are defined on (Section~4.1). Like TPP and TAAM, CAGNN and SAFER receive no task identifier at test time: CAGNN assigns a test graph to the task whose prototype of SGC-smoothed features is nearest, as TPP does, and SAFER to the task with the smallest subspace distance. SAFER's open-set rejection is not ported, so every test graph is assigned to a learned task. Replay budgets are per task.
The datasets are split into training, validation, and test sets in a 6:2:2 ratio. We use the Adam optimizer~\cite{kingma2015adam} for all experiments, and all baseline methods use their original hyperparameter settings.

\noindent\textbf{Hyperparameters of \model.} Table~\ref{tab:hparams} lists every hyperparameter of \model. A single configuration is shared by all seven benchmarks and all three settings (transductive, inductive, and mixed); no value is tuned per dataset. The bottleneck width $b$ is not searched but set by a fixed rule of the input feature dimension $d$, $b=\min(\max(\lfloor d/4\rfloor,16),128)$, which gives $b{=}128$ (CoraFull, $d{=}8710$), 32 (Arxiv, 128), 128 (Reddit, 602), 25 (Products, 100), 114 (Kindle, 457), 61 (DBLP, 244), and 128 (ACM, 1018). The adapter's up-projection is zero-initialized so that the adapter starts as the identity map on $Z_{base}$. The learning rate and weight decay are the same as those of the SGC-based baselines (e.g., Finetuning and EWC), and the number of epochs follows the common protocol of all methods.
\textit{Selection procedure.} We deliberately performed no hyperparameter search for \model: all values were fixed before any experiment was run and were never adjusted on validation or test performance. The optimization settings are inherited from the SGC-based baselines, $b$ follows the dimension rule above, and $K{=}2$ and $M{=}10$ are conventional defaults. Consequently, every number reported for \model comes from a single untuned configuration, whereas the baselines use settings tuned by their original authors; this choice also avoids selection bias from tuning on the evaluation data.
Louvain~\cite{blondel2008fast} partitioning shuffles the node visiting order with the Python random number generator, so the Louvain seed equals the run seed and the reported standard deviations already include the variation of the community partition.

\begin{table}[t]
\centering
\caption{Hyperparameters of \model, shared by all datasets and settings.}
\label{tab:hparams}
\resizebox{\columnwidth}{!}{%
\begin{tabular}{lll}
\toprule
Component & Hyperparameter & Value \\
\midrule
Structural prior & Propagation hops $K$ & 2 ($\hat{A}$ sym.-normalized, self-loops) \\
\midrule
\multirow{3}{*}{TIBA adapter} & Bottleneck width $b$ & $\min(\max(\lfloor d/4 \rfloor, 16), 128)$ \\
 & Residual scale $\alpha$ & 1.0 (up-projection zero-initialized) \\
 & Classifier head & linear, Xavier-uniform init. \\
\midrule
\multirow{4}{*}{Optimization} & Optimizer & Adam, full batch \\
 & Learning rate & 0.01 \\
 & Weight decay & $5\times10^{-4}$ \\
 & Epochs per task $e$ & 500 \\
\midrule
\multirow{3}{*}{Prototypes} & Community detection & Louvain (resolution 1, unweighted) \\
 & Prototypes per task $M$ & 10 (largest communities by \#train nodes) \\
 & Louvain seed & run seed \\
\midrule
Routing & Similarity / vote & cosine max per task; graph-level majority \\
\bottomrule
\end{tabular}}
\end{table}

\section{Algorithm}
\label{sec:algorithm}

Algorithm~\ref{alg:sage_cog} details the training and inference procedures of \model. The framework operates in two phases: continual training with parameter-free projection and task-free inference via dynamic routing.

\noindent\textbf{Phase 1: Continual training.} For each incoming task $t$, the framework first computes the symmetric normalized adjacency matrix $\hat{A}_t$ with self-loops and performs static $K$-hop Laplacian smoothing to extract the parameter-free base embedding $Z_{base}$ (lines 4--6). A task-specific bottleneck adapter $\theta^t$ and classification head $\phi^t$ are then initialized, with the adapter's up-projection layer zero-initialized so that $Z_{adapted}\!=\!Z_{base}$ at the start of training (line 7). The adapter and head are optimized for $e$ epochs via a simple MLP step over the precomputed $Z_{base}$ (lines 8--12), reducing per-epoch cost to matrix multiplications. After training converges, the task graph is partitioned into disjoint communities, and the $M$ communities with the most training nodes $\tilde{C}$ are selected to filter out structural noise (lines 13--14). For each $c_m \in \tilde{C}$, a prototype is obtained by averaging the precomputed base embeddings of its training nodes (lines 16--19), and the resulting prototype set $\mathcal{P}^t$ is stored in the prototype memory $\mathcal{P}$, bypassing the need to retain raw nodes or edges (line 20).

\noindent\textbf{Phase 2: Task-free inference.} For a test graph $G_{test}$ with unknown task identity, the raw features are projected into the structural space through the same parameter-free SGC operation (lines 23--24). This parameter-free projection avoids the feature distortion that would arise from passing a later task's graph through a backbone trained on earlier tasks. \model then computes, for each test node, the maximum cosine similarity against the prototypes of each past task (lines 25--27), and a graph-level majority vote over per-node predictions yields the inferred task identity $\hat{t}$ (line 28). A zero vector keeps norm $0$ after $\ell_2$-normalization, so its cosine similarity with every prototype is $0$ and the node still votes; both $\arg\max$ return the first maximum, i.e., the earliest task. A test node whose embedding is zero therefore votes for the first task. This holds for 212 test nodes on Products, where graph-level routing nevertheless selects the correct task for every test graph. Finally, $Z_{test}$ is dynamically routed to the matching task-specific adapter $\theta^{\hat{t}}$ and classifier head $\phi^{\hat{t}}$ to produce the predicted labels $\hat{Y}$ (lines 29--32).

\begin{algorithm}[!ht]
\caption{Overall Framework of \model}
\label{alg:sage_cog}
\footnotesize
\begin{algorithmic}[1]
\renewcommand{\algorithmicrequire}{\textbf{Input:}}
\renewcommand{\algorithmicensure}{\textbf{Output:}}

\makeatletter
\algnewcommand{\LineComment}[1]{\Statex \hskip\ALG@thistlm $\triangleright$\ \textit{#1}}
\makeatother

\Require A sequence of continual graph tasks $\mathcal{T} = \{G_1, G_2, \ldots, G_T\}$, SGC aggregation hops $K$, the maximum number of prototypes per task $M$
\Ensure Predicted labels $\hat{Y}$ for a given test graph $G_{test}$

\State Initialize prototype memory $\mathcal{P} \leftarrow \emptyset$
\State Initialize adapter memory $\Theta \leftarrow \emptyset$, and classifier head memory $\Phi \leftarrow \emptyset$

\LineComment{\textbf{Phase 1: Continual Training}}
\For{each task $t \in \{1, 2, \ldots, T\}$}
    \State Obtain the current task graph $G_t = (V_t, E_t)$ with node features $X_t$ and labels $Y_t$
    \LineComment{1. Topology-Infused Bottleneck Adaptation}
    \State Compute symmetric normalized adjacency matrix $\hat{A}_t$ from $E_t$ with self-loops
    \State $Z_{base} \leftarrow \hat{A}_t^K X_t$ \Comment{\textit{static precomputation}}
    \State Initialize bottleneck adapter $\theta^t$ (with zero-initialized up-projection) and head $\phi^t$
    \For{$epoch = 1$ to $e$}
        \State $Z_{adapted} \leftarrow Z_{base} + \alpha \cdot \text{Adapter}_{\theta^t}(Z_{base})$
        \State Optimize $\theta^t$ and $\phi^t$ via Cross-Entropy Loss using $Y_t$
    \EndFor
    \State $\Theta \leftarrow \Theta \cup \{\theta^t\}$, $\Phi \leftarrow \Phi \cup \{\phi^t\}$

    \LineComment{2. Community-aware Prototype Extraction}
    \State Execute community detection on $E_t$ to obtain communities $C = \{c_1, \ldots, c_{|C|}\}$
    \State Select at most $M$ communities with the most training nodes, $\tilde{C} \subseteq C$ (communities without training nodes are dropped)
    \State Initialize task prototype set $\mathcal{P}^t \leftarrow \emptyset$
    \For{each $c_m \in \tilde{C}$} \Comment{\textit{parallel}}
        \State $P_m^t \leftarrow \frac{1}{|c_m^{tr}|} \sum_{v \in c_m^{tr}} Z_{base}^{(v)}$ \Comment{\textit{training nodes of $c_m$}}
        \State $\mathcal{P}^t \leftarrow \mathcal{P}^t \cup \{P_m^t\}$
    \EndFor
    \State $\mathcal{P} \leftarrow \mathcal{P} \cup \{\mathcal{P}^t\}$ \Comment{\textit{store prototype memory}}
\EndFor

\LineComment{}
\LineComment{\textbf{Phase 2: Task-free Inference and Dynamic Routing}}
\Function{INFERENCE}{$G_{test} = (\tilde{V}_{test}, E_{test}, X_{test})$, test nodes $V_{test} \subseteq \tilde{V}_{test}$}
    \State $\hat{A}_{test} \leftarrow$ Compute normalized adjacency from $E_{test}$
    \State $Z_{test} \leftarrow \hat{A}_{test}^K X_{test}$ \Comment{\textit{parameter-free projection}}

    \LineComment{1. Task Identification via Majority Voting}
    \For{each node $v \in V_{test}$}
        \State $\hat{t}_v \leftarrow \operatorname*{arg\,max}_{t} \left( \max_{P_m^t \in \mathcal{P}^t} \frac{Z_{test}^{(v)} \cdot P_m^t}{\|Z_{test}^{(v)}\|_2 \|P_m^t\|_2} \right)$
    \EndFor
    \State $\hat{t} \leftarrow \operatorname*{arg\,max}_{t} \big|\{ v \in V_{test} : \hat{t}_v = t \}\big|$ \Comment{\textit{graph-level consensus}}

    \LineComment{2. Dynamic Routing and Classification}
    \State Retrieve task-specific adapter $\theta^{\hat{t}} \in \Theta$ and classifier head $\phi^{\hat{t}} \in \Phi$
    \State $Z_{adapted} \leftarrow Z_{test} + \alpha \cdot \text{Adapter}_{\theta^{\hat{t}}}(Z_{test})$
    \State $\hat{Y} \leftarrow \operatorname*{arg\,max} \text{Head}_{\phi^{\hat{t}}}(Z_{adapted})$ \Comment{\textit{row-wise}}
    \State \textbf{return} $\hat{Y}$
\EndFunction
\end{algorithmic}
\end{algorithm}

\section{Formal Statements and Proofs}
\label{appsec:proofs}

This appendix formalizes the three propositions stated in Section~3 and provides their proofs or proof sketches. The results rely on standard tools (spectral decomposition, the triangle inequality, and Hoeffding's inequality); they are included to justify three design choices of \model: a precomputed low-pass prior, a community-level rather than single-centroid task memory with max-over-prototypes matching, and a graph-level vote over node-wise routing decisions.

\subsection{Proof of Proposition~1 (Laplacian Low-Pass Filtering)}
\label{appsec:proof_laplacian}

Let $\hat{A} = \tilde{D}^{-1/2}\tilde{A}\tilde{D}^{-1/2}$ denote the symmetric normalized adjacency with self-loops of a graph with $n$ nodes. Because $\hat{A}$ is symmetric (and similar, via $\hat{A}=\tilde{D}^{1/2}(\tilde{D}^{-1}\tilde{A})\tilde{D}^{-1/2}$, to the row-stochastic matrix $\tilde{D}^{-1}\tilde{A}$), it admits an orthonormal eigendecomposition $\hat{A} = U \Lambda U^\top$ with real eigenvalues $\lambda_1, \dots, \lambda_{n} \in (-1, 1]$ ordered such that $\lambda_1 = 1$. In the eigenbasis, the graph signal $X$ decomposes as $X = U \tilde{X}$ where $\tilde{X} = U^\top X$ collects the spectral coefficients. Then
\begin{equation}
Z_{base} = \hat{A}^K X = U \Lambda^K \tilde{X} = \sum_{i=1}^{n} \lambda_i^K u_i \tilde{X}_i,
\end{equation}
so the $i$-th spectral component is scaled by $\lambda_i^K$. For $|\lambda_i| < 1$, $\lambda_i^K \to 0$ as $K$ grows; for $\lambda_i = 1$ (the eigenspace spanned by $\tilde{D}^{1/2}\mathbf{1}_{\kappa}$ over connected components $\kappa$; self-loops exclude $\lambda=-1$), the component is preserved exactly. In terms of the normalized-Laplacian frequency $\mu = 1-\lambda \in [0,2)$, the response is $(1-\mu)^K$: it equals one at $\mu=0$ and is below one in absolute value elsewhere, so $Z_{base}$ is a low-pass filter in the sense of~\cite{shuman2013emerging}. The response decreases on $\mu\in[0,1]$; for even $K$ it rises again for $\mu>1$, which is why Proposition~1 is stated in terms of $|\lambda|$. This implies that $Z_{base}$ retains the smooth, homophilous content of $X$ while attenuating high-frequency (rapidly varying) components. The empirical plot in the sensitivity analysis in Section~4.6 confirms this: accuracy saturates at small $K$ (around $K=2$ to $3$) because additional smoothing brings diminishing returns once high-frequency noise has been damped.

\subsection{Proof of Proposition~2 (Within-Task Approximation)}
\label{appsec:proof_community}

Let $z_v = Z_{base}^{(v)}$, let $c_m^{tr}$ be the training nodes of community $c_m \in C$, which partition the node set $S$ over which the prototypes are averaged, $w_m = |c_m^{tr}|/|S|$, $P_m = \frac{1}{|c_m^{tr}|}\sum_{v \in c_m^{tr}} z_v$, and $P_0 = \frac{1}{|S|}\sum_{v \in S} z_v = \sum_m w_m P_m$. Write $\bar{x} = x/\|x\|_2$ for $x \neq 0$, let $\bar{0}$ be any unit vector (then both sides of \eqref{eq:weighted_cos} are $0$ when $z_v = 0$, and $P^\top p = 0$ for every $p$ when $P = 0$, so the steps below hold unchanged), and let $a(v)$ be the community containing $v$. For any unit vector $p$,
\begin{equation}
\|z_v\|_2\,\|\bar{z}_v - p\|_2^2 = 2\|z_v\|_2 - 2\,z_v^\top p .
\label{eq:weighted_cos}
\end{equation}
Averaging \eqref{eq:weighted_cos} over $v$ with the global prototype gives
\begin{equation}
\mathbb{E}_v\!\left[\|z_v\|_2\,\|\bar{z}_v - \bar{P}_0\|_2^2\right] = 2\,\mathbb{E}_v\|z_v\|_2 - 2\|P_0\|_2,
\end{equation}
and with each node's own community prototype
\begin{equation}
\mathbb{E}_v\!\left[\|z_v\|_2\,\|\bar{z}_v - \bar{P}_{a(v)}\|_2^2\right] = 2\,\mathbb{E}_v\|z_v\|_2 - 2\textstyle\sum_m w_m\|P_m\|_2 .
\end{equation}
The difference is $\Delta = 2\big(\sum_m w_m\|P_m\|_2 - \|P_0\|_2\big)$, which is non-negative by the triangle inequality $\|\sum_m w_m P_m\|_2 \le \sum_m w_m\|P_m\|_2$, with equality iff all non-zero $P_m$ are positive multiples of one vector. By \eqref{eq:weighted_cos}, $\bar{P}_m$ also minimizes the weighted distance of community $c_m$ over all unit vectors, so no other single direction per community does better. The tightening is largest when community means point in different directions, which is the regime where community detection is most informative. The inequality holds for any partition of $S$; Louvain supplies one from the task's own edges without fixing the number of groups. Our ablation (the ablation study in Section~4.2, \textit{w/o community-aware} row) is consistent with this: removing community-awareness lowers AP on both benchmarks (by 5.26 points on CoraFull and 1.43 points on Kindle).

\noindent\textit{Remark.} This is an in-sample statement about the task's own nodes, not a bound on the margin to the prototypes of other tasks. Inference-time queries lie outside $S$, so for them the gain equals $\Delta$ only insofar as the test nodes of a community resemble its training nodes, which we evaluate empirically in the task-ID accuracy analysis in Section~4.4.

\noindent\textbf{Proof of Corollary~1 (Top-$M$ truncation).} Let $\bar{z}_v$ and $\bar{P}$ denote $\ell_2$-normalized vectors as in the paper, let $P^\star_v$ be the most similar prototype of $v$ in the full set $\mathcal{P}^t_{C}$, and let $Q^\star_v$ attain $r(v)=\min_{Q\in\mathcal{P}^t}\|\bar{P}^\star_v-\bar{Q}\|_2$. Top-$M$ matching returns the retained prototype with the highest cosine score, so its score is at least $\bar{z}_v^\top\bar{Q}^\star_v$. The loss in score is therefore at most $\bar{z}_v^\top(\bar{P}^\star_v-\bar{Q}^\star_v) \le \|\bar{P}^\star_v-\bar{Q}^\star_v\|_2 = r(v)$ by Cauchy--Schwarz. Since $\|\bar{z}_v-\bar{P}\|_2^2 = 2-2\cos(\bar{z}_v,\bar{P})$ for normalized vectors, the squared distance grows by twice the score loss, so $0 \le \epsilon_M \le 2\,\mathbb{E}_{v \in U}[r(v)]$. 

Table~\ref{tab:topm_bound} reports the measured truncation gap $\hat{\epsilon}_M$ and the bound $2\,\mathbb{E}_v[r(v)]$ at the default $M=10$; the bound holds for every test query on all four benchmarks (queries matched against their own task's prototypes, seed 1024; on Products, 212 test queries and 580 prototypes whose SGC embedding is zero are excluded, since they have no direction).
\input{tables/topm_bound}

\subsection{Proof of Proposition~3 (Majority-Vote Consistency)}
\label{appsec:proof_vote}

Let $t^\star$, $p_t$, and $\gamma$ be as in Proposition~3. \textit{Full vote.} Task $t^\star$ receives $p_{t^\star}|V_{test}|$ votes and every other task at most $(p_{t^\star}-\gamma)|V_{test}|$, so the mode is $t^\star$ when $\gamma>0$ and another task when $\gamma<0$. If $p_{t^\star}>1/2$, every other task receives at most a $1-p_{t^\star}$ share of the votes, so $\gamma \ge 2p_{t^\star}-1>0$. \textit{Sampled vote.} Assume $\gamma > 0$ and draw $N$ nodes uniformly at random with replacement. For a competing task $t$, let $B_i = 1$ if the $i$-th drawn node votes for $t^\star$, $B_i=-1$ if it votes for $t$, and $B_i=0$ otherwise. The $B_i$ are independent, take values in $[-1,1]$, and have mean $p_{t^\star}-p_t \ge \gamma > 0$, whatever the dependence between the votes of neighboring nodes in the graph. Task $t$ ties or beats $t^\star$ only if $\frac{1}{N}\sum_i B_i \le 0$, which by Hoeffding's inequality~\cite{hoeffding1963probability} has probability at most $\exp(-2N(p_{t^\star}-p_t)^2/2^2)\le\exp(-N\gamma^2/2)$. A union bound over the $T-1$ competing tasks shows that the sampled vote returns $t^\star$ with probability at least $1-(T-1)\exp(-N\gamma^2/2)$, which tends to $1$ as $N$ grows.

\noindent\textit{Remark on correlated errors.} Node votes are correlated through the graph structure, since neighboring nodes share smoothed features. The randomness above comes only from sampling, so the bound does not assume independent votes; correlation can change the margin $\gamma$ but does not enter the bound. Section~4.4 measures the edge-wise correlation of routing errors (between 0.253 on Reddit and 0.899 on CoraFull) and shows that graph-level routing still assigns all 294 evaluations to the correct task.

\section{Extended Comparison}
\label{appsec:extended_comparison}

This appendix expands the brief comparison in Section~3.4 with a detailed architectural and computational contrast against the two closest replay-free baselines.

\subsection{Comparison with Existing Replay-Free Methods}
Recent parameter-efficient replay-free methods such as Task Profiling and Prompting (TPP)~\cite{niu2024tpp} and Task-Aware Adaptive Modulation (TAAM)~\cite{liu2026taam} reduce the dependence on historical data but introduce new computational and architectural overheads during task adaptation. TPP prepends learnable tokens to node features that must pass through a frozen GNN backbone at every training epoch, so it does not address the underlying cost of recursive message passing. TAAM uses a frozen SGC backbone but introduces substantial architectural complexity through Neural Synapse Modulators (NSMs), which require multi-head attention on every forward pass. TPP's task prototype is itself parameter-free (a mean of Laplacian-smoothed raw features), but its classifier requires contrastive pre-training of the GNN backbone on the first task; TAAM relies on multi-step diffusion. In contrast, TIBA statically precomputes the structural prior, eliminating epoch-level message passing and complex node-level modulators, and projects the structurally infused features into a lightweight bottleneck adapter that achieves efficient task adaptation without inflating the forward pass. \model further obtains task prototypes directly from the precomputed $Z_{base}$ at no additional training cost, with a single shared embedding for both adaptation and routing. This compressed, topology-infused bottleneck is therefore sufficient to achieve a favorable stability--plasticity trade-off in CGL without the computational overhead of existing methods.

\subsection{Novelty against Single-Centroid Prototypes}
The community-aware prototype design (Section~3.2) differs from existing replay-free approaches in two ways. \textbf{(i)} TPP and TAAM~\cite{niu2024tpp,liu2026taam} also compute parameter-free prototypes, but one per task from all training nodes; our prototypes are computed per community on the same $Z_{base}$ that the adapters use, so routing and adaptation share one embedding. \textbf{(ii)} The multi-prototype set provides a granular density estimate of the task's feature space, so queries drawn from diverse neighborhoods can match any of the task's local modes rather than a single averaged centroid (empirically validated in the task-ID accuracy analysis in Section~4.4).

\subsection{Implication of TIBA}
TIBA does not merely replace the traditional GNN backbone; it redefines graph topology as a persistent, non-drifting anchor for sequential adaptation. Because the parameter-free operator that produces $Z_{base}$ is shared across tasks, the plasticity provided by the BFA is always grounded in the original structural manifold: new tasks add only a bounded perturbation around $Z_{base}$ and cannot deform it. This consistent use of $Z_{base}$ as a common denominator keeps stored prototypes compatible with future query embeddings and provides a reliable foundation for both the community-aware prototype extraction of Section~3.2 and the dynamic routing of Section~3.3.

\subsection{Putting the Three Components Together}
The three components of \model interlock at inference time as follows. During training, each task $t$ produces three artifacts that are stored and never overwritten: a bottleneck adapter $\theta^t$, a linear classifier $\phi^t$, and a prototype set $\mathcal{P}^t$ of at most $M$ prototypes. At inference, a test graph traverses these artifacts in a fixed order. First, it is projected into $Z_{test}$ by the same parameter-free SGC operator used during training (Step~1). Next, the per-node embeddings are matched against the entire prototype memory to produce per-node task votes (Step~2), and those votes are aggregated into a single task identity $\hat{t}$ by graph-level majority voting (Step~3). Only once $\hat{t}$ is fixed does the query enter a task-specific, trainable pathway through $\theta^{\hat{t}}$ and $\phi^{\hat{t}}$ (Step~4). This parameter-free-first, task-specific-last ordering is deliberate: every step that consumes trainable parameters is gated by the output of a parameter-free step, so the trainable adapters only ever receive queries that a parameter-free step has already routed to them. The community-aware prototypes thus play a dual role, as a compact memory that replaces replay and as a discrete router that replaces soft attention, while the TIBA adapters focus purely on per-task plasticity.

\subsection{LwF}
LwF~\cite{li2018lwf}, a distillation-based regularization method, performs like the other regularization methods, with AP of 2.90--7.98\% on the four transductive benchmarks, so the paper does not report it.

\subsection{Choice of Community Detector}
Table~\ref{tab:detectors} replaces Louvain, the default of \model, by four other community detection algorithms on CoraFull (Sec.~4.5). AP is above every baseline in Table~2 (SAFER, the strongest, 95.13\%) with each algorithm. It barely changes because the graph-level vote picks the correct task on every test graph, although per-node routing accuracy ranges from 70.76\% (Infomap) to 75.84\% (Fast-Greedy).

\begin{table}[h]
\centering
\caption{\model on CoraFull with five community detection algorithms (three seeds).}
\label{tab:detectors}
\resizebox{0.8\columnwidth}{!}{
\begin{tabular}{lccc}
\toprule
Algorithm & AP (\%) $\uparrow$ & Graph routing (\%) & Node routing (\%) \\
\midrule
Louvain (default) & 95.78±0.24 & 100.00±0.00 & 75.35±0.12 \\
Leiden & 95.78±0.24 & 100.00±0.00 & 75.12±0.09 \\
Infomap & 95.79±0.22 & 100.00±0.00 & 70.76±0.17 \\
Fast-Greedy & 95.78±0.24 & 100.00±0.00 & 75.84±0.00 \\
Label Propagation & 95.78±0.24 & 100.00±0.00 & 72.90±0.62 \\
\bottomrule
\end{tabular}}
\end{table}

\section{Model-Maintenance Measurements}
\label{app:maintenance}

This appendix reports the measurements summarized in Secs.~\ref{sec:eval_memory}, \ref{sec:eval_routing}, and~\ref{sec:eval_cost}. All runs use an NVIDIA A100 80GB PCIe GPU, PyTorch~2.4.1, and CUDA~12.1.

\smallskip\noindent\textbf{CaT with the true task identifier.} For the task-IL row of Table~\ref{tab:equal_memory}, CaT is trained exactly as in the class-IL row (budget 40, seeds 1024--1026), and at test time its logits are masked to the classes of the true task, using the task-IL evaluation of the CGLB code. Its training, storage, and time are therefore those of the class-IL run; only the prediction scope changes.

\smallskip\noindent\textbf{Cost of adding a task.} Table~\ref{tab:pertask} times each task on Arxiv ($T{=}20$, seed 1024, the setting of Table~\ref{tab:main_results_transductive}), excluding the re-evaluation of earlier tasks. Peak GPU memory is reset at the start of each task and read at its end.

\begin{table}[h]
\centering
\caption{Per-task training time and peak GPU memory on Arxiv (A100, seed 1024).}
\label{tab:pertask}
\resizebox{0.99\columnwidth}{!}{
\begin{tabular}{lccccc}
\toprule
\textbf{Method} & \makecell{\textbf{Mean time}\\\textbf{per task (s)}} & \makecell{\textbf{First}\\\textbf{task (s)}} & \makecell{\textbf{Last}\\\textbf{task (s)}} & \makecell{\textbf{Max peak}\\\textbf{memory (MB)}} & \makecell{\textbf{Last-task peak}\\\textbf{memory (MB)}} \\
\midrule
CaT & 14.93 & 17.15 & 14.58 & 265.87 & 92.77 \\
TPP & 3.37 & 17.62 & 2.47 & 235.57 & 89.87 \\
\model & \textbf{3.15} & \textbf{2.08} & 2.74 & \textbf{170.39} & \textbf{81.34} \\
\bottomrule
\end{tabular}}
\end{table}

\smallskip\noindent\textbf{A long task stream.} We split CoraFull into 70 tasks of one class each and run \model in its default configuration (Appendix~A.2) with seed 1024. Table~\ref{tab:longstream} reports the stored model after $T$ tasks and the routing of one test graph per learned task. Storage grows linearly with $T$, mostly in the adapters, and the routing latency grows with the number of stored prototypes, as Sec.~\ref{subsec:complexity} predicts. At $T{=}70$, routing selects the correct task for 68 of the 70 test graphs. With one class per task, accuracy within the routed task is trivially 100\%, so the table does not report it.

\begin{table}[h]
\centering
\caption{Storage and routing on a 70-task CoraFull stream (one class per task, default configuration, seed 1024).}
\label{tab:longstream}
\resizebox{0.99\columnwidth}{!}{
\begin{tabular}{rccccccc}
\toprule
$T$ & \makecell{\textbf{Adapters}\\\textbf{(MB)}} & \makecell{\textbf{Heads}\\\textbf{(MB)}} & \makecell{\textbf{Prototypes}\\\textbf{(MB)}} & \makecell{\textbf{Total}\\\textbf{(MB)}} & \makecell{\textbf{Routing}\\\textbf{(ms/graph)}} & \makecell{\textbf{Routing}\\\textbf{accuracy (\%)}} \\
\midrule
10 & 89.54 & 0.35 & 3.34 & 93.24 & 3.42 & 100.00 \\
35 & 313.40 & 1.22 & 12.05 & 326.68 & 8.19 & 100.00 \\
70 & 626.81 & 2.44 & 24.07 & 653.32 & 15.18 & 97.14 \\
\bottomrule
\end{tabular}}
\end{table}

\smallskip\noindent\textbf{How many nodes the vote needs.} Graph-level routing votes over all test nodes of a test graph (Sec.~\ref{sec:eval_routing}). Table~\ref{tab:votesize} instead samples $N$ test nodes per test graph with replacement (three draws) and repeats only the vote, with the trained models of seeds 1024--1026. A single node routes correctly in 52.78--98.89\% of the test graphs; ten nodes raise this to 86.48--100\%, and 100 nodes to 98.07--100\%.

\begin{table}[h]
\centering
\caption{Graph-level routing accuracy (\%) with $N$ sampled test nodes per test graph (mean ± std over seeds 1024--1026).}
\label{tab:votesize}
\resizebox{0.85\columnwidth}{!}{
\begin{tabular}{lcccc}
\toprule
\textbf{Dataset} & $N{=}1$ & $N{=}10$ & $N{=}100$ & \textbf{All nodes} \\
\midrule
CoraFull & 72.70±4.50 & 97.78±0.55 & 100.00±0.00 & 100.00±0.00 \\
Arxiv & 52.78±6.73 & 93.89±3.47 & 100.00±0.00 & 100.00±0.00 \\
Reddit & 98.89±0.96 & 100.00±0.00 & 100.00±0.00 & 100.00±0.00 \\
Products & 57.97±6.32 & 86.48±2.21 & 98.07±1.67 & 100.00±0.00 \\
\bottomrule
\end{tabular}}
\end{table}

\end{document}

%% file: section/ch1_introduction.tex
\section{Introduction}


\noindent\textbf{Background.} Graph representation learning (GRL) powers systems that store and serve graph data, such as recommendation~\cite{wang2019neural} and social-network~\cite{freeman2000visualizing} platforms. These graphs are rarely static: their structure and node labels evolve over time, so the learned models must be maintained accordingly. To handle such streams, \textit{continual graph learning} (CGL)~\cite{Febrinanto2023lifelong,zhang2023cglb} learns from a sequence of graph tasks while mitigating \textit{catastrophic forgetting}, the loss of past-task performance when past data is unavailable later.

CGL problems are typically framed as task-incremental (task-IL) or class-incremental (class-IL) learning. In task-IL, the task identifier is provided at inference time, so the model discriminates only among the classes of that task. In class-IL, task identifiers are withheld at inference, forcing the model to discriminate among all classes seen so far; this scenario is regarded as more realistic and significantly more challenging than task-IL~\cite{zhang2023cglb}. Within class-IL, two settings are distinguished by inter-task edges: a \textit{transductive} setting without them, and an \textit{inductive} setting with them.

Existing CGL approaches fall into three families~\cite{liu2023cat,han2024taco,zhang2024continual}. (i) \textit{Regularization}-based methods~\cite{liu2021twp,su2023ssrm} penalize updates to important parameters, but the penalty limits plasticity~\cite{pan2020continual,tang2021graphbased,han2024taco}. (ii) \textit{Memory-replay}-based methods~\cite{zhou2021ergnn,zhang2022ssm,liu2023cat,han2024taco} replay past samples, but they must retain past graph data, raw or condensed, which raises privacy and data-retention concerns~\cite{shaowang2025dataminimization} and caps scalability under strict memory budgets. (iii) \textit{Replay-free} approaches bypass historical data. Early parameter-isolation methods~\cite{Zhang2023hpns,zhang2023pignn} assign disjoint parameter subsets per task; among recent frameworks, TPP~\cite{niu2024tpp} preserves task knowledge via prototypes and prompts, while TAAM~\cite{liu2026taam} does so via neural synapse modulators. Both TPP and TAAM remain bottlenecked, however. TPP's graph neural network (GNN) backbone is pre-trained on the first task and then frozen, preventing it from adapting to later tasks (Sec.~\ref{subsec:complexity}). TAAM freezes its per-task modules, yet nodes of later classes still compete with those of earlier tasks. On Arxiv and Products, it loses accuracy on earlier tasks, with backward transfer (BWT) of $-3.41$ and $-5.67$, and reaches only 82.96\% and 86.94\% average performance (AP), 9.78--12.28 points below our method (Sec.~\ref{main-indivisual}). Our method instead bypasses recursive message passing: it precomputes a parameter-free structural prior and routes it through a lightweight bottleneck adapter, yielding a favorable accuracy--efficiency trade-off.

Empirically, these bottlenecks manifest as setting-specific bias: each prior method is designed and evaluated primarily for a single setting, and none of the compared methods evaluates streams that interleave transductive and inductive tasks in its original paper. In transductive streams, TACO~\cite{han2024taco} runs out of memory on Products and averages only $22.27\%$ AP on the other three benchmarks (Sec.~\ref{main-indivisual}). In inductive streams, the ranking reverses: TPP~\cite{niu2024tpp}, which averages $95.63\%$ AP in the transductive setting, trails TACO ($75.14\%$ for TPP vs.\ $77.77\%$ for TACO, AP averaged over three benchmarks). TPP is also $1.8$--$3.7\times$ slower than our method there (Sec.~\ref{main-indivisual}). This inconsistency motivates \pd{a \textit{unified} evaluation protocol over} transductive, inductive, and mixed task streams.





\noindent\textbf{Our contributions.} We propose \textbf{Co}mmunity-aware \textbf{P}rototype \textbf{E}xtraction for \textbf{C}ontinual \textbf{G}raph \textbf{L}earning (simply, \model), a replay-free CGL framework built on a single parameter-free structural prior that jointly supports task learning and task-free inference. Fig.~\ref{fig:framework} illustrates \model in the class-IL scenario. During training, each task graph is projected via Laplacian smoothing into a base embedding $Z_{base}$, which is used to train a lightweight task-specific adapter and to extract community-aware prototypes $\mathcal{P}^t$. At inference, the same projection yields $Z_{test}$, which is matched against the prototype memory $\mathcal{P}=\{\mathcal{P}^1,\dots,\mathcal{P}^T\}$ of all observed tasks by a majority vote of its nodes (or node by node when a test graph mixes tasks). $Z_{test}$ is then routed to the corresponding adapter and classifier (details in Sec.~\ref{sec:methodology}).

\begin{figure*}[th]
  \centering
  \includegraphics[width=0.80\textwidth]{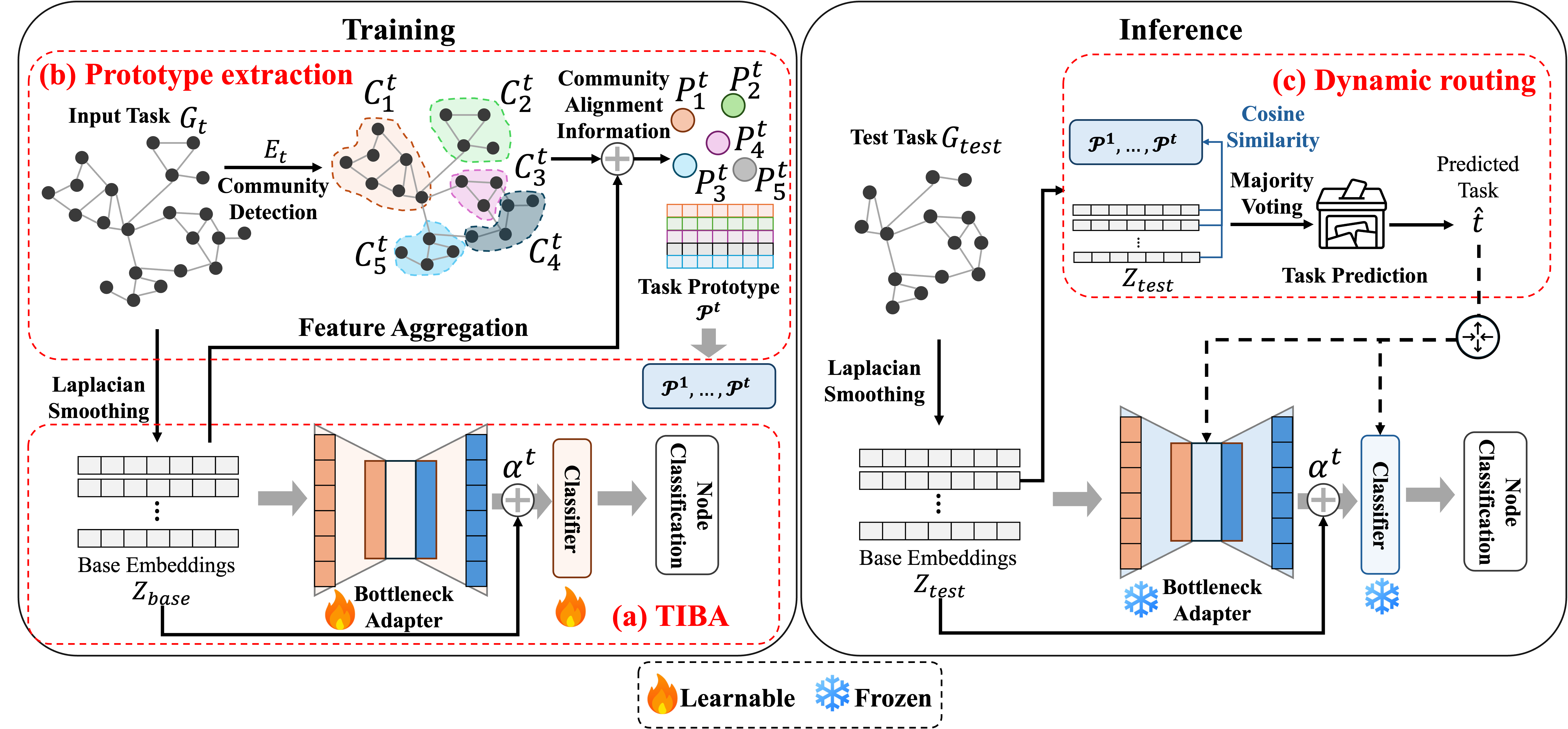}
  \caption{Overview of \model. During training, $K$-hop Laplacian smoothing maps each task graph to $Z_{base}$, on which a task-specific bottleneck adapter is trained and from which community-aware prototypes $\mathcal{P}^t$ are extracted into memory. During inference, $Z_{test}$ is routed to the matching adapter and classifier via graph-level majority voting over the prototype memory, or node by node in a node-level variant.}
  \label{fig:framework}
\end{figure*}

Our main contributions are summarized as follows:
\begin{itemize}
\item \textbf{Community-aware prototypes as a replay-free memory.} To our knowledge, \model is the first replay-free CGL method to use graph communities as the unit of task memory, distilling each task into at most $M$ non-learned community-level prototypes with a constant memory cost per task.
\item \textbf{Fast model maintenance at equal storage.} \model runs end to end $1.3$--$7.9\times$ faster than the most accurate competing baseline on each of the seven benchmarks and $1.1$--$14.0\times$ faster in each of the 35 mixed configurations (seven benchmarks, five task ratios; Sec.~\ref{subsec:complexity}, Secs.~\ref{main-indivisual}--\ref{sec:mixed_res}). It also scales to the largest benchmark (Sec.~\ref{sec:eval_cost}). At equal storage and without task identifiers, it matches the AP that the strongest replay baseline, CaT, reaches only with the true task identifier (within 0.89 percentage points) while running $5.9$--$14.0\times$ faster (Sec.~\ref{sec:eval_memory}).
\item \textbf{A single structural prior for learning, memory, and inference.} All three components of \model, namely Topology-Infused Bottleneck Adaptation (TIBA), community-aware prototype extraction, and graph-level dynamic routing, are organized around one parameter-free structural prior $Z_{base}$, enabling task-free inference without a GNN backbone.
\item \textbf{Unified evaluation across transductive, inductive, and mixed settings.} Against up to eleven baselines on seven benchmarks and interleaved mixed streams, \model outperforms the second-best model by up to $1.41$, $6.32$, and $9.33$ percentage points of AP in the three settings, respectively; it trails TACO or TPP only on DBLP, by at most $2.50$ percentage points.
\item \textbf{Routing beyond single-task test graphs, and robustness.} When one test graph mixes nodes of several tasks, a node-level variant of the routing still gives the best AP among CGL methods (Sec.~\ref{sec:eval_routing}). Under random edge noise up to 50\%, \model has the highest AP in all 16 noisy settings and leads every replay-free baseline at 50\% noise by 1.66--3.57~percentage points of AP (Sec.~\ref{sec:eval_robust}).
\end{itemize}

%% file: section/ch2_problem_definition.tex
\section{Problem Statement}
\label{sec:problem}

A graph $G=(V,E)$ has node set $V$, edge set $E$, and adjacency matrix $A \in \{0,1\}^{|V| \times |V|}$ with $A_{i,j}=1$ iff $(v_i,v_j) \in E$. Each node $v_i$ has a class label $y_i \in \mathcal{C}$, where $\mathcal{C}$ is the set of all classes in the stream, and $X \in \mathbb{R}^{|V| \times d}$ holds the $d$-dimensional node attributes.

We study class-IL: after training on tasks $\mathcal{T}_1,\ldots,\mathcal{T}_t$, the model receives no task identifier at test time and must label each test node from all classes seen so far, while mitigating catastrophic forgetting.

\noindent\textbf{Task regimes.} A CGL model processes a stream of tasks $\mathcal{T} = \{\mathcal{T}_1, \dots, \mathcal{T}_T\}$; at task $\mathcal{T}_t$, it observes a subgraph $G_t = (V_t, E_t)$ with label space $\mathcal{C}_t \subseteq \mathcal{C}$. We denote $V_{<t} = \bigcup_{i<t} V_i$ and $\mathcal{C}_{<t} = \bigcup_{i<t} \mathcal{C}_i$. Following CGLB~\cite{zhang2023cglb} and TACO~\cite{han2024taco}, we consider two regimes and add a third. In the \textit{(i) transductive} regime (traditional class-IL), $V_t \cap V_{<t} = \emptyset$ and inter-task edges are severed. In the \textit{(ii) inductive} regime (generalized class-IL), for $t>1$, earlier nodes may reappear ($V_t \cap V_{<t} \neq \emptyset$), label spaces may overlap ($\mathcal{C}_t \cap \mathcal{C}_{<t} \neq \emptyset$), and inter-task edges are preserved. The \textit{(iii) mixed} regime randomly interleaves the two and better reflects deployment (construction in Sec.~\ref{sec:mixed_construction}; scheduler in Appendix~A.1).

\noindent\textbf{Task-free inference.} At test time, the model receives a test graph $G_{test}=(\tilde{V}_{test},E_{test})$ with features $X_{test}$ and test nodes $V_{test} \subseteq \tilde{V}_{test}$; the stream specifies which nodes and edges form $G_{test}$ but withholds its task identity. Following TPP~\cite{niu2024tpp}, the default protocol assumes one source task $\mathcal{T}_{t^\star}$ with $t^\star \le t$ and $V_{test} \subseteq V_{t^\star}$ ($t$ is written $T$ in Sec.~\ref{sec:methodology}). The test graph $G_{test}$ is $G_{t^\star}$ with the features and edges of all its nodes (training, validation, and test) but no labels; only the nodes of $V_{test}$ are scored and, for methods that infer the task, cast task votes. Sec.~\ref{sec:eval_routing} also evaluates test graphs whose nodes span several tasks, where the task is inferred per node. Under the default protocol, a method that infers the task of the whole test graph predicts among that task's classes, whereas a method with one shared classifier chooses among all classes seen so far; our comparisons with shared-classifier methods (Secs.~\ref{main-indivisual} and~\ref{sec:eval_memory}) are therefore stated under this protocol.

\noindent\textbf{Replay-free setting.} After learning task $\mathcal{T}_{t-1}$, the model may not store any raw nodes, attributes, or labels of earlier tasks and retains knowledge only through its parameters or compressed, non-raw statistics; this restriction makes both intra-task classification (plasticity) and inter-task identification hard. Earlier nodes that reappear in an inductive $G_t$, along with their features and edges, come from the current stream rather than memory, and we use only current training labels; after each task, \model discards the raw graph. Inductive tasks of Kindle, DBLP, and ACM are split by node timestamp~\cite{han2024taco}, and in mixed streams the reappearing nodes of an inductive task are drawn only from earlier tasks (Appendix~A.1), so nothing from the future is visible.

%% file: section/ch3_proposed_method.tex
\section{Methodology}
\label{sec:methodology}

This section presents the three components of \model: (i) TIBA trains a lightweight task-specific adapter on a precomputed structural prior; (ii) community-aware prototype extraction distills each task into community-level prototypes; and (iii) graph-level dynamic routing uses these prototypes to route each unlabeled test graph to the matching adapter and classifier. Appendix~B gives the full procedure as pseudocode.

\subsection{Topology-Infused Bottleneck Adaptation}
\label{sec:sgc_extraction}

\noindent\textbf{Motivation.} Traditional GNNs couple feature transformation with neighborhood aggregation, forcing expensive message passing at every epoch; this repeated propagation is prohibitive in CGL's task-stream setting and a source of structural distortion across tasks. \model therefore decouples propagation from optimization by following Simplified Graph Convolution (SGC)~\cite{wu2019sgc}: it makes multi-hop topology a parameter-free prior and trains a lightweight task-specific adapter on top of it.

\noindent\textbf{TIBA at a glance.} TIBA (Fig.~\ref{fig:framework}~(a)) passes a structural base embedding $Z_{base}$ through our Bottleneck Feature Adapter (BFA) and adds the result, scaled by a fixed weight $\alpha$, back to $Z_{base}$ to yield the task-specific representation $Z_{adapted}$.

\noindent\textbf{Structural prior.} Given a task $G_t = (V_t, E_t)$ with node features $X \in \mathbb{R}^{|V_t| \times d}$, we build the symmetric normalized adjacency matrix with self-loops,
\begin{equation}
\hat{A} = \tilde{D}^{-1/2}\tilde{A}\tilde{D}^{-1/2},
\end{equation}
where $\tilde{A} = A + I$ adds self-loops so each node attends to its own features, and $\tilde{D}$ is the diagonal degree matrix of $\tilde{A}$. This symmetric normalization prevents high-degree hubs from dominating aggregation and stabilizes the spectral range of $\hat{A}$. Rather than propagating messages inside the training loop, we precompute the $K$-hop aggregation once via $K$ sparse matrix--matrix multiplications,
\begin{equation}Z_{base} = \hat{A}^K X,\end{equation}
a fixed Laplacian smoothing~\cite{chung1997spectral} that attenuates every spectral component of the graph signal $X$ except the smoothest one~\cite{shuman2013emerging}:
\begin{proposition}\label{prop:laplacian}
Let $\hat{A}$ have eigenvalues $\lambda_1, \dots, \lambda_{|V_t|} \in (-1, 1]$. Then $Z_{base} = \hat{A}^K X$ attenuates each eigencomponent of $X$ by $\lambda^K$: components with $|\lambda| < 1$ decay geometrically, while the $\lambda = 1$ eigenspace (spanned by $\tilde{D}^{1/2}\mathbf{1}_{\kappa}$ per connected component $\kappa$) is preserved exactly.
\end{proposition}
\noindent The eigendecomposition $\hat{A}=U\Lambda U^\top$ gives $Z_{base}=U\Lambda^K U^\top X$ (full proof in Appendix~C.1); moderate $K$ suffices in practice, with AP flat for $K\ge2$ on Products (Fig.~\ref{fig:params1_1}). A large $K$ is also undesirable: as the components with $|\lambda|<1$ vanish, the rows of $Z_{base}$ on a connected task graph converge to one direction, and so do the community prototypes averaged from them (Sec.~\ref{sec:prototype}). Since $Z_{base}$ is computed once per task and reused in the prototype extraction (Sec.~\ref{sec:prototype}), and the same operator produces the test embeddings for dynamic routing (Sec.~\ref{sec:tf_infer}), it ties the three components together.

\noindent\textbf{Bottleneck adapter.} Using $Z_{base}$ directly lacks task-specific plasticity, so the BFA adds a residual correction:
\begin{equation}
Z_{adapted} = Z_{base} + \alpha \cdot \text{Adapter}_{\theta^t}(Z_{base}),
\end{equation}
where the hyperparameter $\alpha$ (default $1$) scales the adapter output ($\alpha=0$ removes the adapter), and $\text{Adapter}_{\theta^t}(Z) = \sigma(Z W_{down} + \mathbf{1}\beta_{down}^\top)\, W_{up} + \mathbf{1}\beta_{up}^\top$ with $W_{down} \in \mathbb{R}^{d \times b}$, ReLU activation $\sigma$, $W_{up} \in \mathbb{R}^{b \times d}$, and biases $\beta_{down}\in\mathbb{R}^b$, $\beta_{up}\in\mathbb{R}^d$ ($2db+b+d$ parameters). The bottleneck of width $b \le d/4$ (Appendix~A.2) sharply reduces trainable parameters per task, and because $W_{down}, W_{up}$ act on the already-diffused $Z_{base}$, a single forward pass through the BFA suffices per training epoch.

\noindent\textbf{Zero initialization.} To start each new task from the prior itself, we zero-initialize $W_{up}$ and $\beta_{up}$, so the residual block starts as an identity mapping ($Z_{adapted} = Z_{base}$) and training departs from the structural prior only as far as the current task requires. \model then trains only the adapter weights $\theta^t$ and a linear head $\phi^t$, achieving rapid task specialization through the compressed bottleneck.



\subsection{Community-Aware Prototype Extraction}
\label{sec:prototype}

\noindent\textbf{Motivation.} Having established $Z_{base}$ and the BFA, we now summarize each finished task into a compact per-task memory for inference-time routing. Prior methods condense a task into either a single global prototype, as in replay-free methods~\cite{niu2024tpp,liu2026taam}, or randomly sampled exemplars stored for replay~\cite{zhou2021ergnn, liu2023cat}. Graph nodes, however, are homophilic within local neighborhoods but can vary sharply across structurally distinct regions of the same graph. A single global prototype, therefore, collapses distant communities onto one point, and a query can no longer tell which region of the task it resembles. To address this, \model replaces the global prototype with a set of \textit{community-aware prototypes}: it partitions $G_t$ into communities and summarizes each with its own centroid over $Z_{base}$.

\noindent\textbf{Community partitioning.} We identify cohesive subgraphs (Fig.~\ref{fig:framework}~(b)) using Louvain~\cite{blondel2008fast} as the default; Leiden~\cite{traag2019from} and others are supported, and \model is robust to this choice (Sec.~\ref{diverse-community-detection}). The partition uses only the edge set $E_t$, without additional learning, yielding disjoint communities $C=\{c_1, \dots, c_{|C|}\}$ that group structurally proximal nodes.

\noindent\textbf{Prototype aggregation.} For each community $c_m \in C$ with training nodes $c_m^{tr}$ (communities without training nodes are dropped), the community-aware prototype $P_m^t \in \mathbb{R}^d$ averages the base embeddings of the nodes in $c_m^{tr}$:
\begin{equation}
    P_m^t = \frac{1}{|c_m^{tr}|} \sum_{v \in c_m^{tr}} Z_{base}^{(v)},
\end{equation}
where $Z_{base}^{(v)}$ is the row of node $v$ in the base embedding from Sec.~\ref{sec:sgc_extraction}, computed on all nodes of $G_t$ without labels. \model stores the prototypes of up to $M$ communities with the most training nodes, $\tilde{C} \subseteq C$, as $\mathcal{P}^t = \{P_m^t : c_m \in \tilde{C}\}$; $\mathcal{P}^t_C = \{P_m^t : c_m \in C\} \supseteq \mathcal{P}^t$ denotes the means of all communities and is used only in the analysis below. Larger communities yield centroids that are less sensitive to individual outliers, and the cap keeps per-task storage at $Md$ values; Corollary~\ref{cor:topm} below bounds the matching error that this cap introduces. Because each $P_m^t$ is a fixed mean, \pd{not a learned vector}, it is computed in one pass after training without gradient updates, while preserving the local neighborhood semantics of its community.

\noindent\textbf{Novelty against single-centroid prototypes.} Storing only $\mathcal{P}^t$ per task ($\mathcal{O}(Md)$ memory, independent of task-graph size) bypasses raw-data replay buffers while remaining more expressive than a single global prototype (see the drop of the \textit{w/o community-aware} variant in Sec.~\ref{sec:ablation}). A detailed contrast with the single task-wide prototypes of TPP and TAAM~\cite{niu2024tpp,liu2026taam}, including the density-estimate argument examined in Sec.~\ref{sec:eval_routing}, is deferred to Appendix~D.2.

\noindent\textbf{Within-task approximation.} Routing compares $\ell_2$-normalized vectors (Sec.~\ref{sec:tf_infer}), so we state the bound in that space. Let $S$ be the node set over which the prototypes are averaged (the training nodes of $G_t$, which the sets $c_m^{tr}$ partition), $w_m = |c_m^{tr}|/|S|$ the share of community $c_m$, and $P_0^t = \frac{1}{|S|}\sum_{v \in S} Z_{base}^{(v)}$ the single global prototype. We write $\bar{\cdot}$ for the $\ell_2$-normalization ($\bar{0}$ is any unit vector), so that $\bar{z}_v$ is the normalized $Z_{base}^{(v)}$. We measure the distance of node $v$ to a prototype $P$ by the norm-weighted squared distance $\|Z_{base}^{(v)}\|_2\,\|\bar{z}_v-\bar{P}\|_2^2$, averaged over $v \in S$. Within each community, the normalized mean $\bar{P}_m^t$ minimizes this distance over all unit vectors, i.e., it is the norm-weighted mean direction of the community, so the measure matches the prototypes \model actually stores.
\begin{proposition}\label{prop:community}
For the community partition $C$, matching each node $v \in S$ to the prototype of its own community $c_{a(v)}$ instead of the single global prototype $P_0^t$ reduces the average norm-weighted squared distance by $\Delta = 2\big(\sum_{c_m \in C} w_m\|P_m^t\|_2 - \|P_0^t\|_2\big) \ge 0$, with equality iff all non-zero community means point in the same direction.
\end{proposition}
\noindent The proof combines $\|z\|_2\|\bar{z}-p\|_2^2=2\|z\|_2-2z^\top p$ for unit $p$ with the triangle inequality $\|\sum_m w_m P_m^t\|_2\le\sum_m w_m\|P_m^t\|_2$ (Appendix~C.2). The gap $\Delta$ grows as the community means spread apart in direction. $\Delta$ measures how \pd{closely} the community means $\mathcal{P}^t_C$ approximate the training nodes $S$; it does not by itself enlarge the margin to the prototypes of other tasks, which routing accuracy (Sec.~\ref{sec:eval_routing}) and the \textit{w/o community-aware} variant (Sec.~\ref{sec:ablation}) measure.

\noindent\textit{Remark (from Proposition~\ref{prop:community} to the routing rule).} Let $\bar{z}_v$ and $\bar{P}_m$ denote the $\ell_2$-normalized $Z_{base}^{(v)}$ and $P_m^t$. Three facts connect the proposition to the matching rule of Sec.~\ref{sec:tf_infer}. (i) For normalized vectors, $\|\bar{z}_v-\bar{P}_m\|_2^2 = 2-2\cos(\bar{z}_v,\bar{P}_m)$, so maximizing cosine similarity is equivalent to minimizing squared Euclidean distance. (ii) Since $\mathcal{P}^t_C$ contains each node's own community mean, $\min_{P \in \mathcal{P}^t_C}\|\bar{z}_v-\bar{P}\|_2^2 \le \|\bar{z}_v-\bar{P}_{a(v)}\|_2^2$, so, because the weights $\|Z_{base}^{(v)}\|_2$ are non-negative, nearest-prototype matching over $\mathcal{P}^t_C$ reduces the distance by at least $\Delta$. (iii) Storing only $\mathcal{P}^t$ raises each node's best squared distance by a non-negative amount, whose average over a node set $U$ is \begin{equation} \epsilon_M = \mathbb{E}_{v \in U}\Big[\min_{P \in \mathcal{P}^t}\|\bar{z}_v-\bar{P}\|_2^2 - \min_{P \in \mathcal{P}^t_{C}}\|\bar{z}_v-\bar{P}\|_2^2\Big] \ge 0; \end{equation} $\epsilon_M$ vanishes when every community is retained. The following corollary bounds this term by the distance from each node's best prototype in the full set to the retained set. \begin{corollary}[Top-$M$ truncation]\label{cor:topm} For a node $v$, let $P^\star_v$ be its most similar prototype in $\mathcal{P}^t_{C}$, and let $r(v)=\min_{Q\in\mathcal{P}^t}\|\bar{P}^\star_v-\bar{Q}\|_2$ be the distance from $\bar{P}^\star_v$ to the nearest retained prototype. Truncation increases the squared distance of $v$ to its best match by at most $2r(v)$, so $0 \le \epsilon_M \le 2\,\mathbb{E}_{v \in U}[r(v)]$ for any node set $U$. \end{corollary} \noindent The best retained score is at least $\bar{z}_v^\top\bar{Q}$ for the $Q$ attaining $r(v)$, so by Cauchy--Schwarz the score drops by at most $r(v)$ and the squared distance by at most $2r(v)$ (Appendix~C.2), and Sec.~\ref{sec:eval_cost} measures $\epsilon_M$ on the test nodes ($U = V_{test}$) of the transductive benchmarks. Because the corollary bounds each node separately, the corresponding increase in the norm-weighted distance of Proposition~\ref{prop:community} is at most $2\,\mathbb{E}_{v \in S}[\|Z_{base}^{(v)}\|_2\, r(v)]$; we report the unweighted $\epsilon_M$.

The resulting prototype memory $\mathcal{P} = \{\mathcal{P}^1, \dots, \mathcal{P}^T\}$ is what the graph-level dynamic routing described next in Sec.~\ref{sec:tf_infer} uses.

\subsection{Task-Free Inference and Dynamic Routing}
\label{sec:tf_infer}

\noindent\textbf{Overview.} The final component combines the adapters of TIBA (Sec.~\ref{sec:sgc_extraction}) with the prototype memory (Sec.~\ref{sec:prototype}) to label test graphs whose task identity is unknown. As shown in Fig.~\ref{fig:framework}~(c), \model computes the parameter-free base embeddings $Z_{test}$ of the test graph, matches each test node against the prototype memory $(\mathcal{P}^1, \dots, \mathcal{P}^T)$ by cosine similarity, elects a task identity $\hat{t}$ by graph-level majority voting, and routes $Z_{test}$ through the adapter $\theta^{\hat{t}}$ and head $\phi^{\hat{t}}$ to produce labels $\hat{Y}$.

\noindent\textbf{Step 1: Projection.} We project the test graph $G_{test}$ into the structural embedding space via the same $K$-hop Laplacian smoothing as in Sec.~\ref{sec:sgc_extraction}: $Z_{test} = \hat{A}_{test}^K X_{test}$, and keep the rows of the test nodes $V_{test}$. The operator has no trained parameters and uses no labels, so queries lie in the prototypes' feature space. Under the default protocol, $G_{test}$ is the task graph supplied again by the stream, so a query is computed exactly like a base embedding of that task.

\noindent\textbf{Step 2: Prototype matching.} Let $\mathcal{P} = \{\mathcal{P}^1, \dots, \mathcal{P}^T\}$ be the set of community-aware prototypes collected from all $T$ observed tasks. For a test node $v \in V_{test}$ with embedding $Z_{test}^{(v)}$, we score its similarity to task $t$ by its maximum cosine similarity to $t$'s prototypes:
\begin{equation}
    s(v, t) = \max_{P_m^t \in \mathcal{P}^t} \frac{Z_{test}^{(v)} \cdot P_m^t}{\|Z_{test}^{(v)}\|_2 \|P_m^t\|_2}.
\end{equation}
With the maximum (rather than an average) over $\mathcal{P}^t$, a query from \textit{any} retained community of $t$, not only from one near the task mean, can be attributed to $t$; Corollary~\ref{cor:topm} covers the dropped communities.

\noindent\textbf{Step 3: Majority voting.} A single node's prediction is susceptible to local noise, so \model lets the test nodes vote on the task. Each node $v$ casts a vote for its most similar task, $\hat{t}_{v} = \arg\max_{t} s(v, t)$, and the graph-level task identity $\hat{t}$ is the mode of the votes over $V_{test}$ (a very large test graph may vote over $N$ nodes drawn uniformly at random; all reported results vote over all test nodes):
\begin{equation}
    \hat{t} = \arg\max_{t} \big|\{ v \in V_{test} : \hat{t}_v = t \}\big|.
\end{equation}
A cosine similarity involving a zero vector is set to $0$, and both $\arg\max$ steps break ties toward the earliest task (Appendix~B). This graph-level aggregation transforms noisy per-node predictions into a single reliable task assignment with no additional parameters. Assigning one task per test graph follows TPP~\cite{niu2024tpp} (Sec.~\ref{sec:problem}), which matches the test graph's mean embedding to the nearest task prototype instead of voting. On test graphs that merge two tasks, Section~\ref{sec:eval_routing} also evaluates a node-level variant without this vote. 

\noindent\textbf{Consistency of majority voting.} A plurality suffices:
\begin{proposition}\label{prop:vote}
Let $t^\star$ be the correct task, $p_t$ the fraction of test nodes in $V_{test}$ whose own vote $\hat{t}_{v}$ is task $t$, and $\gamma = p_{t^\star} - \max_{t \neq t^\star} p_t$ the plurality margin. The vote over all nodes returns $t^\star$ if $\gamma > 0$ and another task if $\gamma < 0$; in particular, $p_{t^\star} > 1/2$ suffices. If $\gamma > 0$, the vote over $N$ nodes drawn uniformly at random with replacement returns $t^\star$ with probability at least $1 - (T-1)\exp(-N\gamma^2/2)$.
\end{proposition}
\noindent The proof (a counting argument and a Hoeffding bound with a union bound over the competing tasks) is given in Appendix~C.3. Sampling is the only source of randomness, so the bound needs no independence between neighboring nodes. With many candidate tasks, the margin can be positive even when $p_{t^\star} \le 1/2$: $p_{t^\star} > 1/2$ holds in 254 of the 294 transductive test-graph evaluations, but $\gamma > 0$ holds in all 294, so every vote is correct (Sec.~\ref{sec:eval_routing}).

\noindent\textbf{Step 4: Routing.} With $\hat{t}$ inferred, \model routes $Z_{test}$ to the matching task-specific components: the selected BFA yields $Z_{adapted} = Z_{test} + \alpha \cdot \text{Adapter}_{\theta^{\hat{t}}}(Z_{test})$, and the task-specific linear head gives the labels $\hat{Y} = \arg\max \text{Head}_{\phi^{\hat{t}}}(Z_{adapted})$, row-wise over the classes of task $\hat{t}$.

\noindent\textbf{Putting the three components together.} Steps~1--3 use no trained parameters, and Step~4 uses only those of the \pd{voted task}. \pd{A} new task adds its own prototypes, adapter, and head and never changes those of earlier tasks. An earlier task \pd{thus loses} accuracy only when another task's prototypes win its vote (Proposition~\ref{prop:vote}; Appendix~D.4), and BWT is 0.00 wherever the vote is correct (Sec.~\ref{main-indivisual}).


\subsection{Efficiency Analysis}\label{subsec:complexity}
Let $|V_t|$ and $|E_t|$ be the numbers of nodes and edges of task graph $G_t$, $d$ the feature dimension, $b$ the bottleneck width, $n_c=\max_t|\mathcal{C}_t|$ the largest number of classes of a task, $K$ the number of hops, $M$ the maximum number of prototypes per task, $e$ the number of epochs, and $T$ the number of observed tasks.

\noindent\textbf{Training.} Per task, smoothing costs $\mathcal{O}(K|E_t|d)$ and is paid once outside the epoch loop; adapter and head training costs $\mathcal{O}(e|V_t|d(b+n_c))$; Louvain partitioning~\cite{blondel2008fast} empirically costs $\mathcal{O}(|E_t|\log|V_t|)$, and prototype aggregation costs $\mathcal{O}(|V_t|d)$. A message-passing GNN with $L$ layers of width $h$ instead pays $\mathcal{O}(eL(|E_t|h+|V_t|dh))$, because propagation is repeated in every epoch. 

\noindent\textbf{Inference.} For a test graph with $|E_{test}|$ edges and $N_{test}=|V_{test}|$ test nodes, projection costs $\mathcal{O}(K|E_{test}|d)$, prototype matching $\mathcal{O}(N_{test}TMd)$, and routing plus classification $\mathcal{O}(N_{test}d(b+n_c))$, so the matching cost grows linearly in $T$. 

\noindent\textbf{Memory.} \model never stores raw nodes, attributes, edges, or labels; at inference, it reads the queried graph from the data system, as any method must, whereas replay keeps an extra copy. Each task adds one adapter, one head, and at most $M$ prototypes, so with 32-bit floating-point (FP32) parameters, the task-growing storage after $T$ tasks is \begin{equation}\label{eq:memory} \mathrm{Mem}(T) \le 4T\big[(2db+b+d) + (dn_c+n_c) + Md\big]\ \text{bytes}, \end{equation} where the three terms are the adapter, the head, and the prototypes, with equality when every task has $n_c$ classes and keeps $M$ prototypes. The storage added per task is therefore bounded by a constant, and for fixed $b$, $n_c$, and $M$, the total grows as $\mathcal{O}(Td)$, as in TPP and TAAM~\cite{niu2024tpp,liu2026taam}. Sec.~\ref{sec:eval_memory} compares this footprint with equal-size replay buffers, and Sec.~\ref{sec:eval_cost} breaks down the time.

%% file: section/ch4_experiments.tex
\section{Performance Evaluation}
\label{sec:eval}

\subsection{Experimental Settings}

\subsubsection{Dataset Construction}
\label{sec:mixed_construction}
We evaluate \model under the three regimes formalized in Sec.~\ref{sec:problem} (transductive, inductive, and mixed). For datasets without timestamps, we adopt the class-IL paradigm of CGLB~\cite{zhang2023cglb}: the continuum is partitioned by node class, and each task introduces new, non-overlapping classes. After each task, every method is evaluated on the test nodes of all tasks so far.

\smallskip\noindent\textbf{Benchmarks.} We use seven public graph benchmarks: four transductive (CoraFull, Arxiv, Reddit, and Products from the CGLB benchmark protocol~\cite{zhang2023cglb}) and three inductive (Kindle, DBLP, and ACM from the TACO protocol~\cite{han2024taco}). Mixed-setting streams are constructed on top of both groups using the randomized scheduler of Appendix~A.1. Per-dataset statistics are given in Table~\ref{data_table}, and detailed feature descriptions in Appendix~A.1.

\input{tables/dataset}

\subsubsection{Comparison Methods and Evaluation Metrics}
We compare \model against eleven baselines spanning three CGL families, namely (i) three regularization-based (EWC~\cite{Kirkpatrick2017ewc}, MAS~\cite{Aljundi2018mas}, TWP~\cite{liu2021twp}), (ii) four replay-based (ER-GNN~\cite{zhou2021ergnn}, SSM~\cite{zhang2022ssm}, CaT~\cite{liu2023cat}, TACO~\cite{han2024taco}), and (iii) four parameter-isolated replay-free (TPP~\cite{niu2024tpp}, TAAM~\cite{liu2026taam}, CAGNN~\cite{xiao2026cagnn}, SAFER~\cite{zhang2026safer}). We do not compare with Pi-GNN~\cite{zhang2023pignn} and HPNs~\cite{Zhang2023hpns}, whose parameter-isolation and prototype-based designs are represented by the four more recent replay-free baselines above. We also include fine-tuning as a lower bound and joint training as an offline reference. We assess performance using two standard metrics~\cite{chaudhry20219continual}: AP$\uparrow$ and BWT$\uparrow$; formal definitions appear in Appendix~A.2. We additionally report wall-clock runtime from model construction to the final evaluation, including each method's own preprocessing. All results are the mean $\pm$ standard deviation across three independent runs; a reported standard deviation of 0.00 denotes a value below 0.005, and Appendix~A.2 details the seeds.

\subsubsection{Implementation}
Baselines are run with the CaT~\cite{liu2023cat} codebase (built on CGLB~\cite{zhang2023cglb}) or with each method's official code, ported to the same pipeline using each paper's optimal configuration. \model uses one untuned configuration for all benchmarks and settings ($K{=}2$, $M{=}10$, $\alpha{=}1$), with bottleneck width $b=\min(\max(\lfloor d/4\rfloor,16),128)$.
All baselines except TWP, SAFER, and CAGNN use a two-layer SGC~\cite{wu2019sgc} with 256-dim hidden units as the backbone (frozen after pre-training in TPP~\cite{niu2024tpp}), since TAAM is built on a frozen SGC and stronger backbones would confound strategy with capacity. TWP, SAFER, and CAGNN keep the backbones they are defined on. TWP~\cite{liu2021twp} penalizes changes to the graph attention network (GAT)~\cite{velivckovic2018graph} parameters that are important for its loss and its attention-based aggregation. CAGNN~\cite{xiao2026cagnn} uses its own dendritic graph convolutions. SAFER~\cite{zhang2026safer} trains task modules on a two-layer graph convolutional network (GCN)~\cite{kipf2017semi} that is pre-trained with GRACE~\cite{zhu2020grace} on the first task and then frozen. Each of these backbones adds an activation between its graph convolutions, so keeping them does not favor \model. For \model, we train for 500 epochs. All methods in Tables~\ref{tab:main_results_transductive} and~\ref{tab:main_results_inductive} and the scalability study (Sec.~\ref{sec:eval_cost}) are run and timed with the same script on one NVIDIA A100 (80GB) server with PyTorch 2.4.1. The noise and runtime-breakdown experiments are run on an NVIDIA RTX A4000, and some rows of the equal-memory comparison (Sec.~\ref{sec:eval_memory}) are run on an A100, so running times are compared only among methods in the same table or figure. Replay-based baselines use per-task memory budgets of 100, 200, 300, and 400 nodes on CoraFull, Arxiv, Reddit, and Products, and 100, 200, and 300 nodes on Kindle, DBLP, and ACM (4,000 nodes in total on Arxiv after 20 tasks). Codebase URLs are listed in Appendix~A.2.

\subsection{Experimental Results}

\subsubsection{Main Results on the Individual Settings}
\label{main-indivisual}

Tables~\ref{tab:main_results_transductive} and~\ref{tab:main_results_inductive} show the results in the transductive and inductive settings, respectively. 

\smallskip\noindent
\textbf{(1) Transductive setting.}
As shown in Table~\ref{tab:main_results_transductive}, under the single-task test-graph protocol, \model achieves the highest AP on all four datasets, ahead of the second-best baseline by 0.65 points on CoraFull (SAFER), 1.41 on Arxiv (TPP), 0.04 on Reddit (TAAM), and 0.22 on Products (CAGNN). Its BWT is 0.00 across all benchmarks. The Reddit and Products margins are about one standard deviation, so on these two datasets \model matches the strongest baseline in AP while running $2.0\times$ and $6.9\times$ faster, respectively. It even surpasses joint training, which learns all tasks with one shared classifier and thus remains exposed to inter-task interference, whereas \model isolates each task in its own adapter and head. Joint training is therefore an offline reference, not an upper bound: after each task, one SGC and one classifier over all classes seen so far are trained further on the training nodes of all tasks to date. The two also differ in their label sets: once a test graph is routed, \model chooses only from that task's classes, whereas joint training chooses from all classes seen so far. TACO is competitive in the inductive setting (Table~\ref{tab:main_results_inductive}) but reaches only 16.93\% AP on CoraFull here, as its replay relies on a continuously coarsened graph buffer.

\smallskip\noindent
\textbf{(2) Inductive setting.}
As shown in Table~\ref{tab:main_results_inductive}, \model achieves the highest AP among CGL methods on Kindle and ACM, exceeding the second-best baseline by 5.93 and 6.32 points. On DBLP, it trails TACO by 1.21 points and TPP by 0.73 points but keeps a BWT of 0.00 and runs $7.9\times$ and $3.7\times$ faster than them, respectively.

\smallskip\noindent
\textbf{(3) Robust efficiency.}
Beyond accuracy, \model runs faster end to end than the most accurate competing baseline on every benchmark (Tables~\ref{tab:main_results_transductive}--\ref{tab:main_results_inductive}), from $1.3\times$ against TPP on Arxiv to $7.9\times$ against TACO on DBLP, and avoids the out-of-memory error that TACO hits on Products. ER-GNN and SSM, which train on a small fixed replay buffer, are faster on the four transductive benchmarks and on Kindle, and TAAM is about 8\% faster on Products, but each of them trails \model there by 4.17--57.75 points of AP.
The two most recent replay-free baselines, CAGNN and SAFER, are $2.1$--$15.8\times$ slower than \model across all benchmarks.

\smallskip\noindent
\textbf{(4) Performance visualization.} Accuracy-matrix heatmaps on Products (Fig.~\ref{fig:visualization}) show that the per-task accuracy of \model does not fade, while the replay-based baselines CaT and SSM show fading columns, the signature of catastrophic forgetting. TAAM's columns also fade, although its per-task modules are frozen, because it predicts across all classes seen so far, allowing later classes to compete with earlier ones for nodes from earlier tasks.

\input{tables/main_result_transductive}

\input{tables/main_result_inductive}

\begin{figure*}[t]
  \centering
  \includegraphics[width=0.9\textwidth]{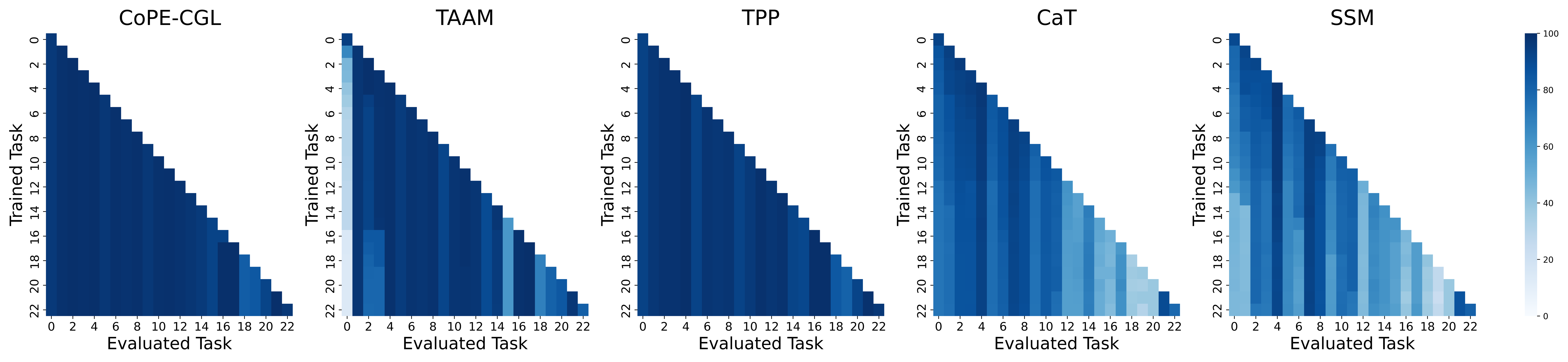}
  \caption{Per-task accuracy matrices on Products (row: after training task $t$; column: task $j$). Darker is higher.}
  \label{fig:visualization}
\end{figure*}

\subsubsection{Main Results on the Mixed Settings}
\label{sec:mixed_res}
Table~\ref{tab:results_mixed} reports the mixed-setting results across seven datasets, varying the transductive-to-inductive task ratio from $1{:}9$ to $9{:}1$. \model achieves the best AP in 33 of the 35 configurations, the shortest runtime in 34, and a BWT of zero or above in 33. Because each adapter and head is frozen after its task, BWT can change only through routing, that is, when later tasks alter which task wins the vote for an earlier test graph. The two exceptions in AP are DBLP at $1{:}9$ and $3{:}7$, where TPP exceeds \model by 1.47 and 2.50 points (89.51\% and 89.03\% for TPP vs.\ 88.04\% and 86.53\% for \model), respectively. However, TPP takes $1.93\times$ and $2.02\times$ as long (38.08 and 39.29~s for TPP vs.\ 19.70 and 19.46~s for \model). The only runtime exception is Products at $1{:}9$, where TAAM is faster (543.15~s for TAAM vs.\ 653.70~s for \model) but has lower AP (88.12\% vs.\ 95.45\%, in the same order). Negative BWT occurs only on Kindle at $3{:}7$ ($-4.44$) and DBLP at $9{:}1$ ($-0.08$, within one standard deviation of zero).

\input{tables/mixed_setting_ratio1}

\subsubsection{Ablation Studies}\label{sec:ablation}
Table~\ref{ablation} reports ablations on CoraFull (transductive) and Kindle (inductive) with four reduced variants, each removing one component (Laplacian smoothing, bottleneck, community-aware prototypes, or majority voting). Without smoothing, $K=0$ and $Z_{base}=X$; without the vote, one task is still chosen per test graph, correctly for 32--33 of the 35 CoraFull test graphs. Removing Laplacian smoothing lowers AP by 2.87 points on CoraFull and causes the largest AP drop on Kindle (40.11 points), so Kindle relies most on the structural prior $Z_{base}$. Removing the bottleneck yields the fastest variant (36.04~s on CoraFull and 4.84~s on Kindle), since it trains no adapter, but it lowers AP by 6.04 points on Kindle. Removing community-aware prototypes lowers AP by 5.26 points on CoraFull and 1.43 points on Kindle. Removing majority voting causes the largest drop on CoraFull (6.24 points) and lowers AP by 8.21 points on Kindle.
\input{tables/ablation}

\subsection{Memory Efficiency at Matched Storage}
\label{sec:eval_memory}

Like TPP and TAAM, \model adds at most a constant amount of storage per task (Eq.~\eqref{eq:memory}). We therefore compare methods at matched final storage on Arxiv and CoraFull (Table~\ref{tab:equal_memory}). We fix \model's bottleneck to $b=13$, below the default width given in Appendix~A.2, and set each baseline's memory knob (replay budget, prompt size, or rank) to the value whose final storage is closest to that of \model. The parameter-free smoothing operator is excluded from the storage count of all methods, and replay storage counts the stored FP32 node features and their structure.

At about 0.4\,MB on Arxiv and 48\,MB on CoraFull, \model outperforms the strongest replay baseline, CaT, by 33.48 and 15.82 points, keeps BWT at 0.00, and runs $5.9\times$ and $14.0\times$ faster. Most of this gap is prediction scope: CaT chooses among all classes seen so far, whereas \model chooses within the task its vote selects (Sec.~\ref{sec:problem}). Given the true task identifier (task-IL), CaT reaches 94.28\% and 95.82\%, within 0.89 points of \model, which infers the task itself. Table~\ref{tab:equal_memory} splits the methods by forgetting. The replay methods forget at this size (BWT from $-4.97$ to $-16.49$), whereas the parameter-isolation methods keep BWT within 1.43 points of zero. A replay buffer of this size holds at most 40 raw nodes per task, whereas \model spends the same number of bytes on an adapter and on the prototypes of the largest communities in each task graph. \model also leads every parameter-isolation baseline on both datasets, by 0.60 points over TPP on Arxiv and 0.43 over SAFER on CoraFull, and is the fastest method (TAAM, SAFER, CAGNN, and TPP on CoraFull were timed on an A100 with seeds 1042--1044). At equal storage, its adapter and prototypes buy more accuracy than prompts, subspaces, or low-rank modulators. CAGNN is the only baseline shown in its default configuration; its storage is far above this budget (825.66\,MB and 3004.06\,MB, about $2{,}000\times$ and $63\times$ that of \model). This setting favors CAGNN: its storage is set by its dendritic network's width rather than by a memory knob like those above, keeping its default capacity. It still trails \model by 1.84 and 1.06 points.

\input{tables/equal_memory}

\subsection{Routing Accuracy and Mixed-Task Test Graphs}
\label{sec:eval_routing}

\subsubsection{Correlated Routing Errors and Task-ID Accuracy}
Proposition~\ref{prop:vote} requires only that the correct task win the vote, not that node errors be independent. Table~\ref{tab:routing_corr} measures how these errors correlate on the four transductive benchmarks. They are positively correlated along edges, up to 0.899 on CoraFull, which is consistent with neighboring nodes sharing smoothed features. Graph-level routing picks the correct task in all 294 cases (98 test graphs, three runs), even in the 40 cases where at most half of the nodes select it on their own; Arxiv accounts for 16 of these 40 cases (out of its 60). The vote needs enough nodes: with 10 test nodes sampled per test graph, routing accuracy is 86.5--100\% on the four benchmarks, and 98.1--100\% with 100 (Appendix~E). The default protocol is the easier case, since each test graph keeps all nodes and edges of its task graph (Sec.~\ref{sec:problem}); Sec.~\ref{sec:mixed_task} drops this assumption.

\input{tables/routing_correlation}

Fig.~\ref{fig:task_id_acc} reports task-ID accuracy under a $5{:}5$ transductive-to-inductive ratio in the mixed setting, where historical nodes reappear in later tasks and task boundaries overlap. Baselines that rely on a single task representation lose accuracy on some datasets; for example, TAAM drops to $74.33\%$ on Kindle. \model identifies tasks nearly perfectly on CoraFull, Arxiv, and DBLP.

\begin{figure}[t]
  \centering
  \includegraphics[width=0.70\linewidth]{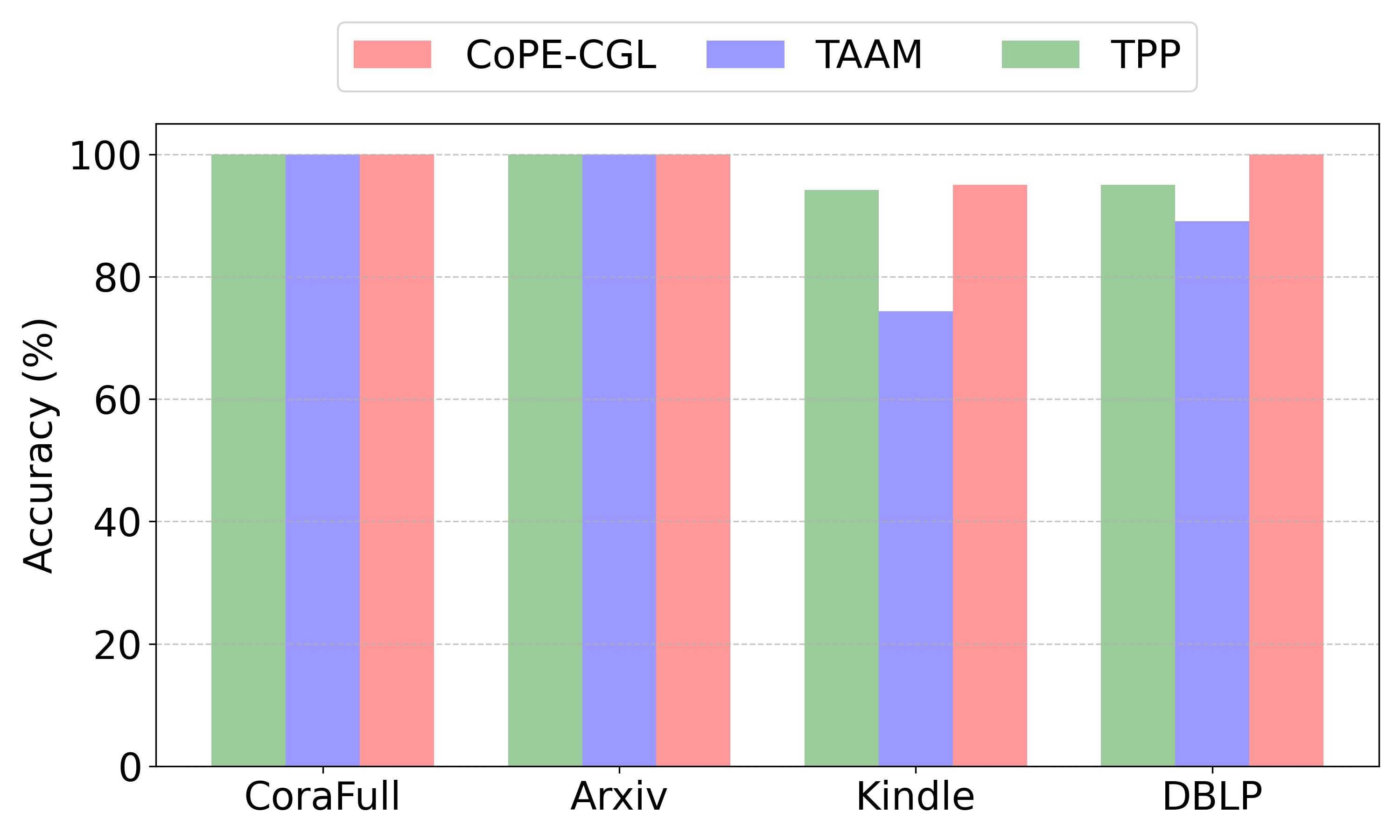}
  \caption{Task-ID accuracy under the $5{:}5$ mixed setting.}
  \label{fig:task_id_acc}
\end{figure}

\subsubsection{Mixed-Task Test Graphs}\label{sec:mixed_task}
Under the default protocol, all nodes of a test graph come from one task (Sec.~\ref{sec:problem}). We now evaluate a harder and more practical case that no parameter-isolation baseline tests in its own paper: the nodes of one test graph come from several tasks and keep the edges between them, as when a deployed model receives a graph snapshot that spans tasks. TPP~\cite{niu2024tpp} and TAAM~\cite{liu2026taam} infer one task per test graph, CAGNN~\cite{xiao2026cagnn} is given the task identity, and SAFER~\cite{zhang2026safer} routes each node but is never tested on a graph that joins nodes of different tasks. For each pair of consecutive tasks, we merge their test nodes into one graph and restore the edges between them that exist in the original full graph. We evaluate a node-level variant of \model's routing that skips the graph-level vote: each node is assigned to the task whose prototypes it matches best among all tasks observed so far, and is classified by that task's adapter and head. For a fair comparison, TPP, TAAM, SAFER, and CAGNN are evaluated using the same node-level routing, while the remaining baselines, which lack task inference, use their single shared classifier. For these four methods, node-level routing is our adaptation. This setting works against \model: it removes the graph-level vote, which costs \model 6.24 points on CoraFull and 8.21 points on Kindle in the ablation (Table~\ref{ablation}), whereas SAFER routes nodes by design.

Excluding joint training, \model attains the best AP on mixed-task test graphs (Mixed AP) on both Arxiv (55.86\%) and Reddit (96.59\%), ahead of the strongest baseline by 15.74 points (CaT) and 23.72 points (TPP), as shown in Table~\ref{tab:mixed_node}. From single-task (Native) to mixed-task (Mixed) test graphs, \model loses 39.04 AP points on Arxiv and 2.99 on Reddit, less than any other parameter-isolation method. TPP and TAAM each summarize a task with a single representation (TPP with a spectral profile and a prompt, TAAM with a task embedding), whereas \model keeps up to $M$ community prototypes per task, giving each node several candidate matches within its own task. On Arxiv, the remaining loss comes from routing rather than classification: even on single-task test graphs, only 55.46\% of nodes select the correct task on their own, against 97.10\% on Reddit (Table~\ref{tab:routing_corr}). Joint training needs no routing, which is why it scores higher than \model on the Arxiv mixed graphs (69.33\%).

\input{tables/mixed_node_routing}

\subsection{Robustness to Graph Structure}
\label{sec:eval_robust}

\smallskip\noindent\textbf{Noisy edges.} Both the smoothing operator and the communities of \model are computed from observed edges, so edge noise propagates into its features, prototypes, and routing. To every task graph of two transductive and two inductive benchmarks, we add uniformly random node pairs as undirected edges, numbering 10--50\% of that graph's edge list, and compare with the replay-free baselines TPP, TAAM, SAFER, and CAGNN (Table~\ref{tab:noisy}). \model has the highest AP in 19 of the 20 settings, including all 16 noisy ones; only TPP leads on clean DBLP (by 0.95 points). From 0\% to 50\% noise, \model loses 5.09 points on average, less than TPP (7.72), SAFER (7.28), and CAGNN (8.32). Only TAAM loses less on average (4.85), but this reflects an exceptional case rather than robustness: one of its tasks on Arxiv fails in the clean run, so its Arxiv AP is higher with noise than without. At 50\% noise, \model still leads on every dataset by 1.66--3.57 points.

\input{tables/noisy_edges}

\smallskip\noindent\textbf{Weak community structure.} Proposition~\ref{prop:community} ties prototype quality to community structure, so we test whether tasks with weaker structure obtain lower AP. Table~\ref{tab:modularity} correlates per-task AP with per-task modularity $Q$ over all tasks of the seven benchmarks. \model shows no significant linear correlation on any benchmark (Pearson $p>0.05$). Where community structure does matter, it matters for the baselines as well. On Products and DBLP, the rank correlation is positive and significant for \model ($\rho=0.49$ and $0.85$) and for three of the baselines on each. On Kindle, three of the four baselines show a significant linear correlation, although it comes from only five tasks and differs in sign across methods. Strongly heterophilic graphs remain outside our scope.

\input{tables/modularity_q}

\smallskip\noindent\textbf{Choice of community detector.}
\label{diverse-community-detection}
On CoraFull, \model reaches 95.78--95.79\% AP with each of five community detection algorithms, above every baseline, so our default Louvain can be replaced without loss (Appendix~D.6).

\subsection{Cost Breakdown, Scalability, and Sensitivity}
\label{sec:eval_cost}

\smallskip\noindent\textbf{Where the time goes.} Table~\ref{tab:runtime} decomposes \model's wall-clock time on all seven benchmarks. The components are timed in separate profiling runs, and their times add up to the end-to-end time of those runs. On five benchmarks, adapter training dominates (71.9--87.4\%). On the two largest graphs, Reddit and Products, community detection dominates instead (70.1\% and 69.6\%), because Louvain runs on the CPU over task graphs with millions of edges. This cost is paid once per task during training; inference runs no community detection and routes with the stored prototypes. Adding a task therefore costs one smoothing pass, one adapter and head training run, and one partition of the new task graph, and leaves the stored components of earlier tasks untouched. On Arxiv (A100, one seed), a new task takes 3.15\,s of training on average at a peak GPU memory of 170\,MB, against 3.37\,s and 236\,MB for TPP and 14.93\,s and 266\,MB for CaT. On a 70-task CoraFull stream with one class per task, storage grows linearly with the number of tasks, and routing one test graph takes 15.2\,ms at $T{=}70$ with 97.1\% routing accuracy (Appendix~E). Routing and classification by the final model take at most 0.5\% of the time. Another 1.3--13.9\% is spent re-evaluating all earlier tasks after each new task; only AP and BWT need this, not deployment.

\input{tables/runtime_breakdown}

\smallskip\noindent\textbf{Scalability.} Fig.~\ref{fig:scale_test} shows how \model, TPP, and TAAM scale on Products as the data volume grows from $20\%$ to $100\%$. We compare with TPP and TAAM because they are the fastest replay-free baselines on Products (510.75 and 398.78 seconds in Table~\ref{tab:main_results_transductive}); SAFER and CAGNN already take $2.1\times$ and $6.9\times$ as long as \model on the full data. On AP (Fig.~\ref{fig:scale_test}(a)), \model leads both baselines at every scale; with only $20\%$ of the data, it already exceeds TPP's AP at $100\%$ (95.88\% for \model\ at $20\%$ vs.\ 95.42\% for TPP). On runtime (Fig.~\ref{fig:scale_test}(b)), \model is faster than TPP at every scale. TAAM is marginally faster than \model at $100\%$ but trails it by 9.78 points in AP.

\begin{figure}[t]
 \begin{subfigure}[t]{0.40\linewidth}
   \centering
   \includegraphics[width=.99\textwidth]{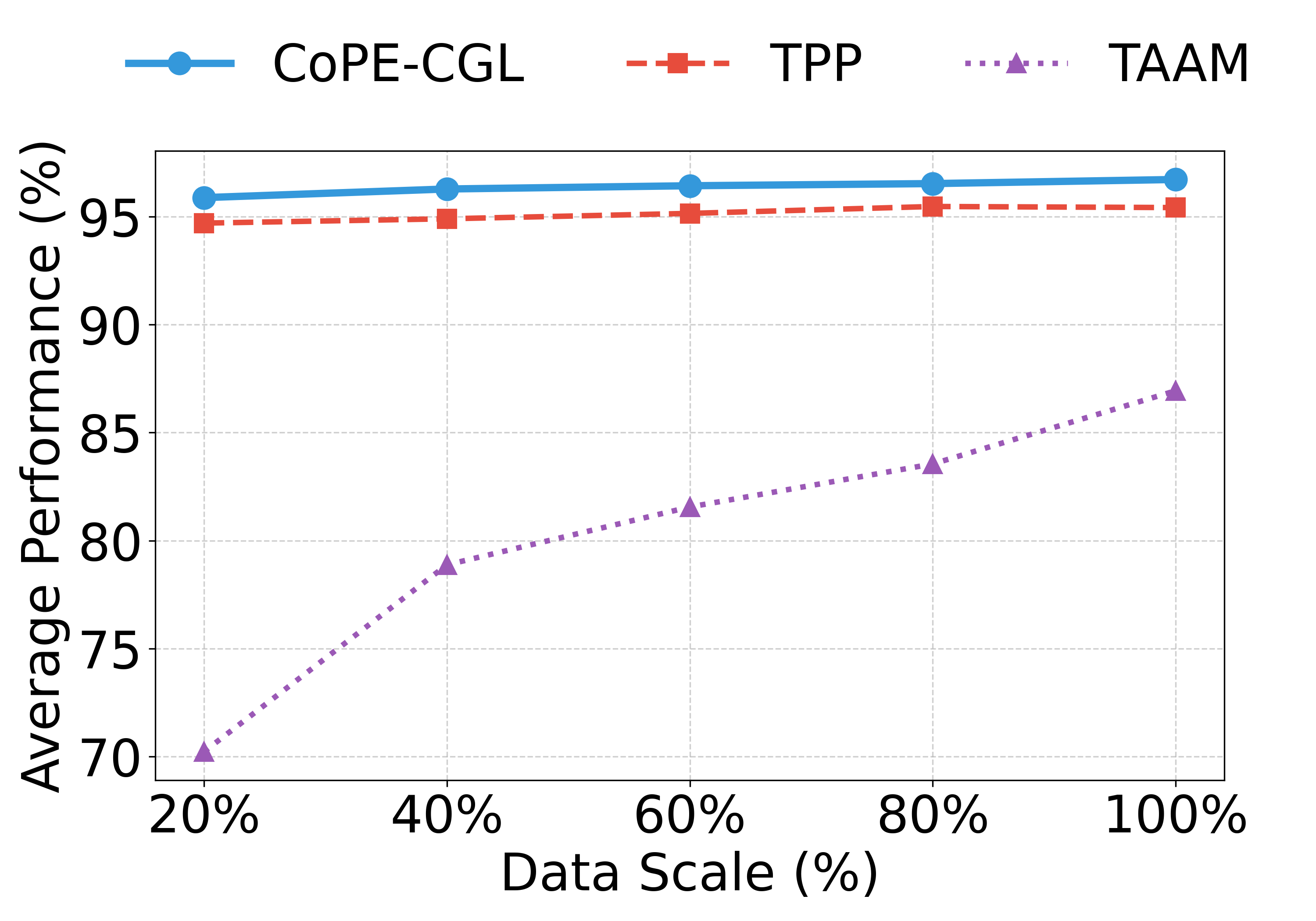}
   \caption{AP}
   \label{fig:scale-a}
 \end{subfigure}
 \begin{subfigure}[t]{0.40\linewidth}
   \centering
   \includegraphics[width=.99\textwidth]{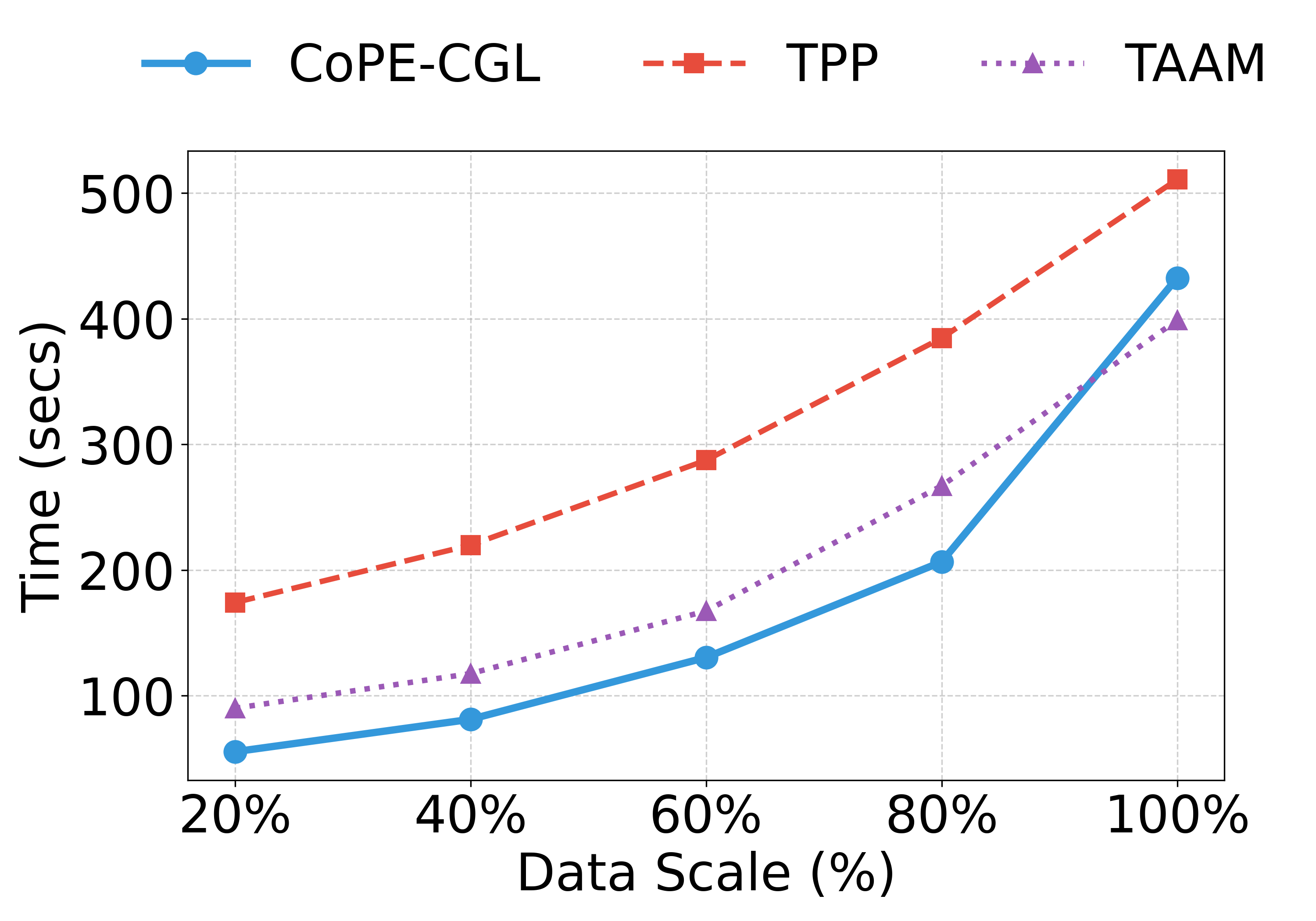}
   \caption{Running time}
   \label{fig:scale-b}
 \end{subfigure}
 \caption{Scalability of \model, TPP, and TAAM on Products: (a) AP, (b) wall-clock running time. The $100\%$ points are from Table~\ref{tab:main_results_transductive}; the others are means over seeds 1024--1026.}
 \label{fig:scale_test}
\end{figure}

\smallskip\noindent\textbf{Parameter sensitivity.} Fig.~\ref{fig:params1} shows the sensitivity of \model's AP to three hyperparameters on Products. For both the number of smoothing hops $K$ and the adapter scaling weight $\alpha$, \model outperforms TPP and TAAM at every tested value except $\alpha=0$, where the adapter is removed and \model is on par with TPP. For the prototype count $M$, AP plateaus for $6 \le M \le 19$, a range that contains our default of $M=10$.

\begin{figure*}[t]
\centering
 \begin{subfigure}[t]{0.23\linewidth}
   \centering
   \includegraphics[width=\textwidth]{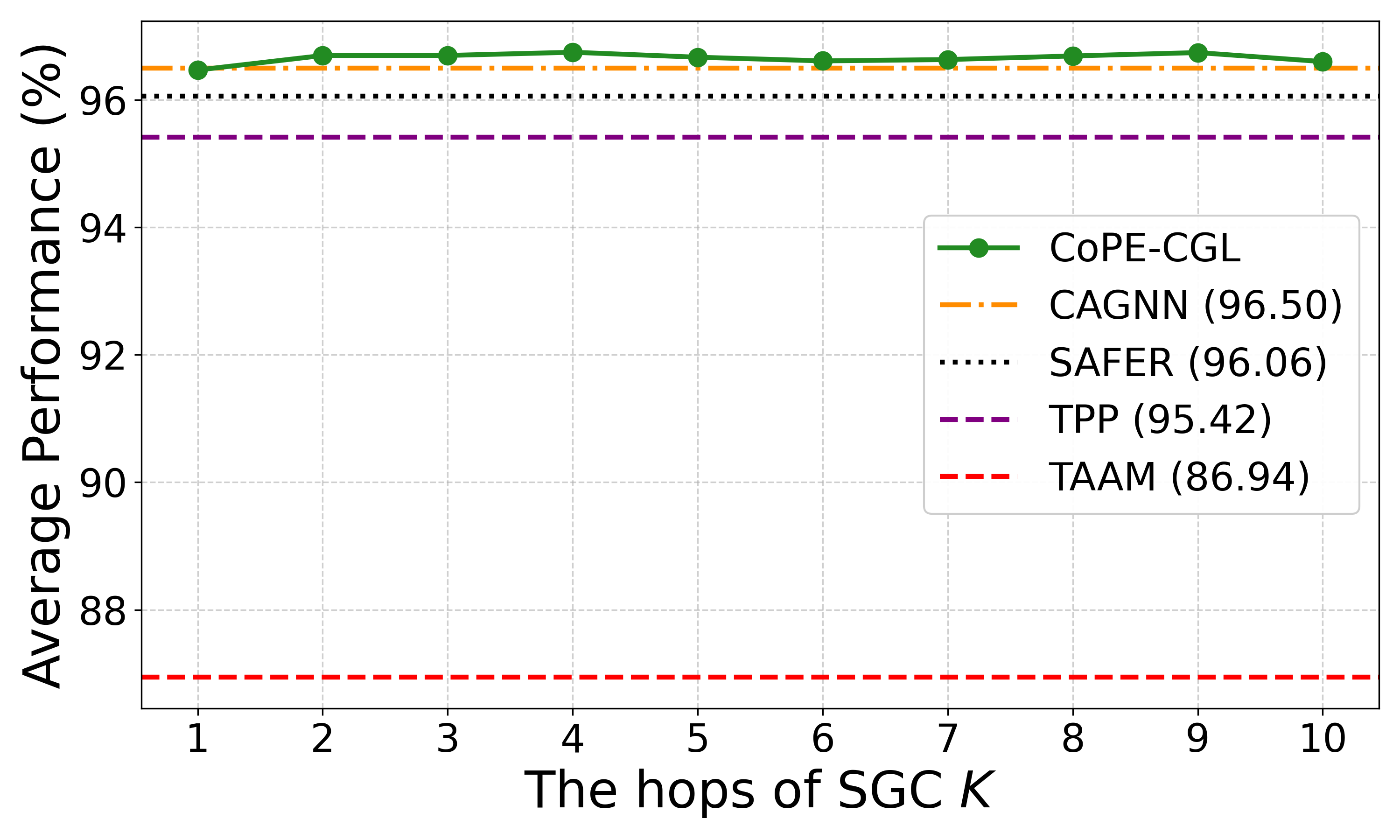}
   \caption{$K$-hops}
   \label{fig:params1_1}
 \end{subfigure}
 \hfill
 \begin{subfigure}[t]{0.23\linewidth}
   \centering
   \includegraphics[width=\textwidth]{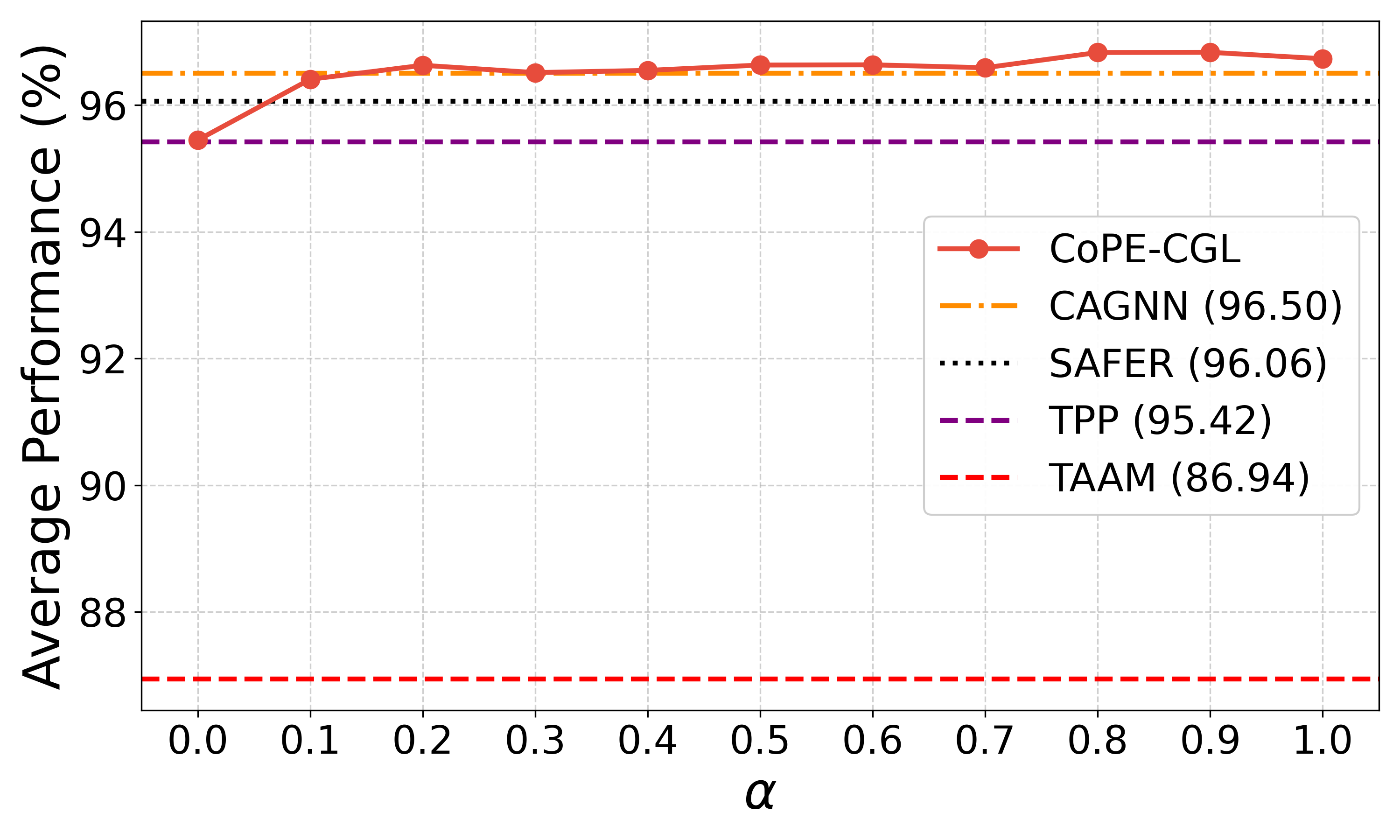}
   \caption{$\alpha$}
   \label{fig:params1_2}
 \end{subfigure}
 \hfill
 \begin{subfigure}[t]{0.23\linewidth}
   \centering
   \includegraphics[width=\textwidth]{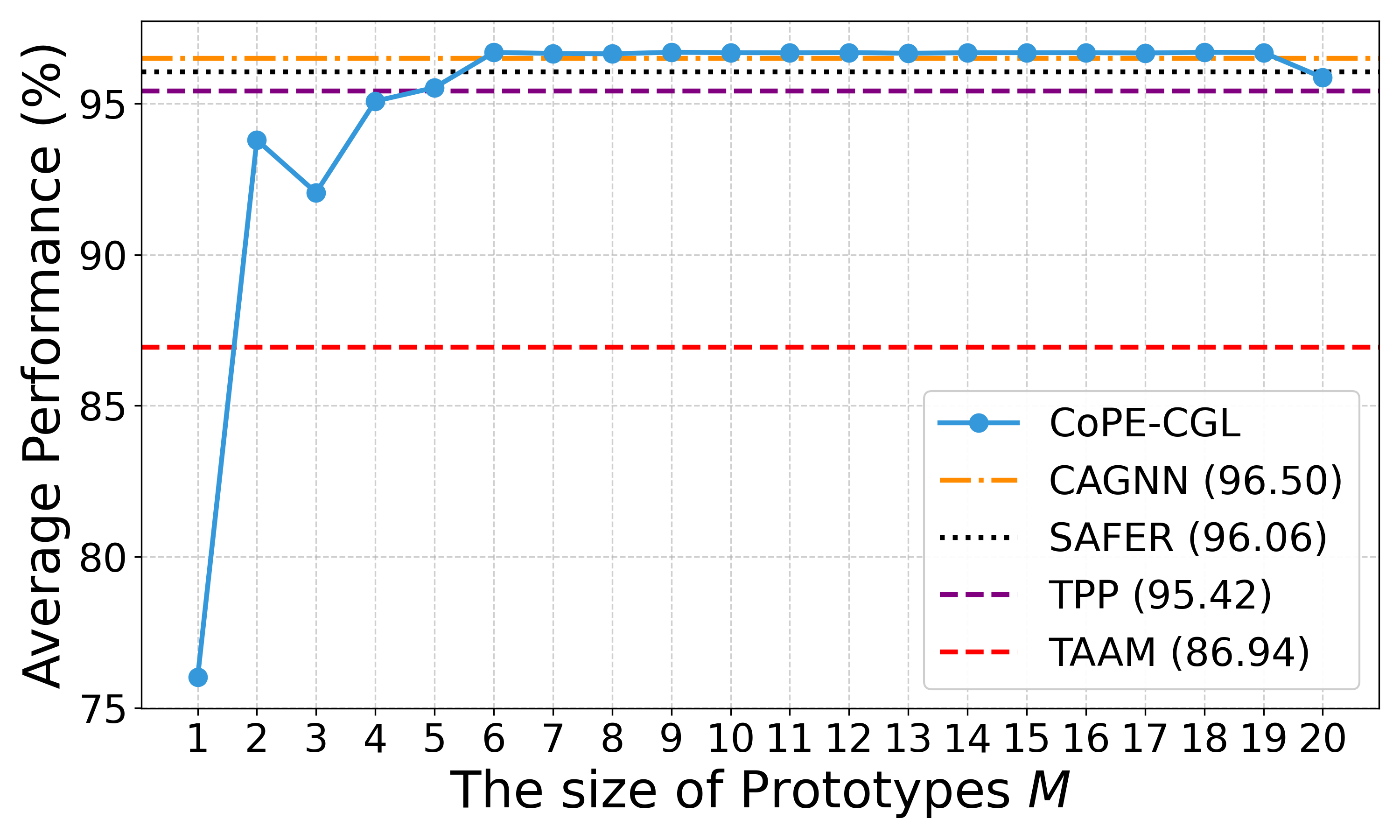}
   \caption{Prototype count $M$}
   \label{fig:params1_3}
 \end{subfigure}
 \caption{Sensitivity of \model's AP on Products to three hyperparameters: smoothing hops $K$, adapter scaling weight $\alpha$, and prototype count $M$. Horizontal lines mark CAGNN, SAFER, TPP, and TAAM at their default settings (Table~\ref{tab:main_results_transductive} means). $\alpha$ is a fixed weight. \model curves are single runs (seed 1024); horizontal lines are three-seed means.}
 \label{fig:params1}
\end{figure*}

\smallskip\noindent\textbf{Cost of keeping only $M$ prototypes.} Corollary~\ref{cor:topm} bounds how far keeping only the $M$ largest communities can move a query from its best prototype. At $M=10$, the measured truncation error $\epsilon_M$ on the four transductive benchmarks ranges from 0.57\% (Reddit) to 17.32\% (Products) of the mean pairwise squared distance between community prototypes (Appendix~C.2 lists all four benchmarks). This is consistent with the AP plateau for $6\le M\le 19$ in Fig.~\ref{fig:params1_3}.

%% file: tables/dataset.tex
\begin{table}[!ht]
\centering
\caption{Dataset statistics; rows with the subscript $_{mixed}$ (e.g., CoraFull$_{mixed}$) are the mixed-setting streams.}
\resizebox{0.82\columnwidth}{!}{
\begin{tabular}{c|c|c|c|c|c}
\hline
 \textbf{Dataset}   & \textbf{\# Nodes} & \textbf{\# Edges} & \textbf{\# Features}  & \textbf{\# Classes} & \textbf{\# Tasks}  \\ 
 \hline
 \hline
 CoraFull & 19,793 & 130,622 & 8,710 & 70 & 35 \\
 Arxiv & 169,343 & 1,166,243 & 128 & 40 & 20 \\
 Reddit & 227,853 & 114,615,892 & 602 & 40 & 20 \\
 Products & 2,449,028 & 61,859,036 & 100 & 46 & 23 \\
 Kindle & 38,450 & 458,018& 457& 6 & 5 \\
 DBLP & 54,265 & 2,751,530& 244& 4 & 10 \\
 ACM & 77,130 & 476,572& 1,018& 4 & 10 \\
 CoraFull$_{mixed}$ & 19,793 & 130,622 & 8,710 & 70 & 32 \\
 Arxiv$_{mixed}$ & 169,343 & 1,166,243 & 128 & 40 & 20 \\
 Reddit$_{mixed}$ & 227,853 & 114,615,892 & 602 & 40 & 20 \\
 Products$_{mixed}$ & 2,449,028 & 61,859,036 & 100 & 46 & 20 \\
Kindle$_{mixed}$ & 38,450 & 458,018& 457& 6 & 12 \\
 DBLP$_{mixed}$ & 54,265 & 2,751,530& 244& 4 & 10 \\
 ACM$_{mixed}$ & 77,130 & 476,572& 1,018& 4 & 10 \\
 \hline

\end{tabular}
}
\label{data_table}
\end{table}

%% file: tables/main_result_transductive.tex
\begin{table*}[t]
\centering
\Huge
\caption{Transductive results: AP, BWT, and wall-clock running time. Bold: best AP excluding Joint. OOM: out of memory.}

\resizebox{\textwidth}{!}{
\fontsize{50}{52}\selectfont
\renewcommand{\arraystretch}{1.1}
\setlength{\tabcolsep}{14pt}
\begin{tabular}{llcccccccccccc}
\toprule
\multirow{2}{*}{\textbf{Category}} & \multirow{2}{*}{\textbf{Methods}} 
& \multicolumn{3}{c}{\textbf{CoraFull}} 
& \multicolumn{3}{c}{\textbf{Arxiv}} 
& \multicolumn{3}{c}{\textbf{Reddit}} 
& \multicolumn{3}{c}{\textbf{Products}} \\
\cmidrule(lr){3-5} \cmidrule(lr){6-8} \cmidrule(lr){9-11} \cmidrule(lr){12-14}
 & & AP (\%) $\uparrow$ & BWT (\%) $\uparrow$ & Time (secs) $\downarrow$
   & AP (\%) $\uparrow$ & BWT (\%) $\uparrow$ & Time (secs) $\downarrow$
   &AP (\%) $\uparrow$ & BWT (\%) $\uparrow$ & Time (secs) $\downarrow$
   &AP (\%) $\uparrow$ & BWT (\%) $\uparrow$ & Time (secs) $\downarrow$
   \\
\midrule

\multirow{1}{*}{Lower bound} 
& Finetuning &  3.21±1.12&   -95.01±0.31&104.58±5.81&4.99±0.11
    &  -96.51±0.11 &71.56±7.56& 4.96±0.28&-99.31±0.56&477.47±6.06 &7.03±0.09&-92.21±0.27&984.67±9.22
 \\
\midrule

\multirow{3}{*}{Regularization} 
& EWC  &  2.94±0.45&-94.03±2.77&363.52±8.63& 4.94±0.01 & -96.87±0.17 & 127.05±2.36&5.10±0.10 & -99.4±0.11 &270.7±4.25& 14.18±0.23& -92.54±0.45&430.00±1.42\\
& MAS  &2.84±0.21 & -94.13±0.75&162.34±1.62& 4.93±0.53 & -94.53±1.13 &80.62±0.45& 10.26±2.52 & -93.43±2.72 & 208.65±1.62&9.45±1.21 & -89.51±1.71&362.92±0.12\\
& TWP  &  20.90±2.23 &  -61.47±0.82 & 650.71±10.65 & 6.07±1.84 & -51.24±1.92 & 295.19±7.73 & 2.11±2.63	& -54.49±4.39& 	2698.21±1.06&  3.19±2.12& 	-39.30±2.46	& 4811.18±4.77\\
\midrule

\multirow{4}{*}{Replay}
& ER-GNN &   72.34±0.25 & -7.45±0.12 & 45.11±1.32 & 65.63±0.31  &-11.72±0.24 &15.18±2.35 &88.91±0.18  &-1.38±0.05 & 38.81±3.92&38.97±0.29  &-15.84±0.36&51.29±1.77
\\
& SSM     &  75.02±0.20&  -5.27±0.10& 49.91±0.12& 66.91±0.27 & -10.94±0.33& 14.87±2.28  & 95.43±0.09&  -0.26±0.02 & 39.24±3.85  &61.08±0.34 & -11.21±0.41 &50.73±1.84
\\
& CaT     & 82.14±0.18 &  -3.81±0.08 &376.21±5.32  &  67.32±0.24  & -10.19±0.28 &  115.83±2.87  &96.85±0.11 &  -0.33±0.03  &1404.66 ±10.25  & 69.73±0.26 &  -5.71±0.22
 & 1868.35±15.61  \\
 & TACO & 16.93±0.29  & 0.00±0.00  & 448.51±8.12  &
8.87±0.23   & 0.00±0.00 &  3241.26±23.74  &
41.02±0.38 &  0.00±0.00  & 29112.47±79.13 &\multicolumn{3}{c}{OOM Error}\\
\midrule

\multirow{5}{*}{Parameter-isolation}
& TPP\footnotemark &  93.78±0.33 & 0.00±0.00 & 75.61±1.32 & 93.83±0.10 & 0.00±0.00 & 48.12±0.77  &99.50±0.01 & 0.00±0.00 & 425.79±2.08 & 95.42±0.13 & 0.00±0.00 & 510.75±0.52 \\
& TAAM & 89.14±0.29 & -0.09±0.01 &	154.24±0.10 &	82.96±1.09 &	-3.41±0.01 &	61.63±0.02 &	99.56±0.03 &	0.00±0.00 &	344.73±1.16 &	86.94±1.95 &	-5.67±0.01 &	398.78±1.26 \\
& CAGNN & 94.90±0.04 & 0.00±0.00 & 481.15±49.65 & 93.05±0.03&0.00±0.00 & 278.54±16.91&99.54±0.01&0.00±0.00 & 2126.62±206.82&96.50±0.10&0.00±0.00&2999.70±220.82\\
& SAFER & 	95.13±0.16& 0.00±0.00 & 216.46±11.18 & 93.26±0.06 &0.00±0.00& 132.31±1.74 & 99.55±0.01 & 0.00±0.00& 551.74±170.74 & 96.06±0.02 &0.00±0.00& 889.09±3.60\\
& \model (Ours)  & \textbf{95.78±0.24} &	0.00±0.00 &	51.96±2.61 &	\textbf{95.24±0.07} &	0.00±0.00 &	35.70±0.35 &	\textbf{99.60±0.02} &	0.00±0.00 &	168.77±12.69 &	\textbf{96.72±0.14} &	0.00±0.00 &	432.47±65.62
 \\
\midrule

\multirow{1}{*}{Full dataset} 
& Joint &  83.31±0.06 &  -2.45±0.05 & 195.13±1.26  &
61.36±0.04  & -17.33±0.07 &   156.79±0.65&
97.95±0.00  & -0.51±0.01 & 2372.42±14.15  &
64.85±0.17 & -9.99±0.35 & 9079.45±40.66
\\
\bottomrule
\end{tabular}
}
\label{tab:main_results_transductive}

\end{table*}
\footnotetext{All methods use the same splits and class order; the original TPP paper reports 85.4 (Arxiv) and 94.0 (Products) under its own splits.}

%% file: tables/main_result_inductive.tex
\begin{table*}[t]
\centering

\caption{Main results in the inductive setting: AP, BWT, and wall-clock running time. Bold: best AP excluding Joint.}

\resizebox{0.99\textwidth}{!}{
\fontsize{32}{34}\selectfont
\begin{tabular}{llcccccccccccc}
\toprule
\multirow{2}{*}{\textbf{Category}} & \multirow{2}{*}{\textbf{Methods}} 
& \multicolumn{3}{c}{\textbf{Kindle}} 
& \multicolumn{3}{c}{\textbf{DBLP}} 
& \multicolumn{3}{c}{\textbf{ACM}}  \\
\cmidrule(lr){3-5} \cmidrule(lr){6-8} \cmidrule(lr){9-11} 
 & & AP (\%) $\uparrow$ & BWT (\%) $\uparrow$ & Time (secs) $\downarrow$
   & AP (\%) $\uparrow$ & BWT (\%) $\uparrow$ & Time (secs) $\downarrow$
   &AP (\%) $\uparrow$ & BWT (\%) $\uparrow$ & Time (secs) $\downarrow$
   \\
\midrule

\multirow{1}{*}{Lower bound} 
& Finetuning & 67.42±10.13  &  -17.61±10.42   & 1.20±0.03&
65.73±7.86    & -19.34±6.58   &1.43±0.02& 
59.28±8.74   &  -17.92±8.89  &1.12±0.01
 \\
\midrule

\multirow{3}{*}{Regularization} 
& EWC  &  75.41±7.92   & -9.62±8.03  & 60.29±0.12 &
77.26±4.43  &  -7.94±3.75  & 65.16±0.04&
64.73±6.01  &  -11.28±5.84 &71.77±0.23
\\
& MAS  &  70.58±6.84  &  -7.12±6.87 &   62.5±0.10  & 
72.33±3.92  &  -5.89±3.12  &  102.5±0.03  &
59.84±5.12  &  -8.76±4.02&  106.92±0.19\\
& TWP  &  77.42 ±4.38  &  -8.12 ±4.54  & 108.73±0.34&
78.31 ±3.48    &-7.36 ±3.01 &  187.41±0.11&
64.12 ±6.84  & -12.08 ±6.52&192.63±0.27\\
\midrule

\multirow{4}{*}{Replay}
& ER-GNN &  51.27±1.33 & 5.75±1.14 & 3.83±0.02 & 81.98±0.32 &  1.40±0.08  & 30.38±0.91  &
64.74±0.41 &  3.11±0.15 &  31.53±1.02
\\
& SSM     &  68.66±0.82 & 5.05±1.20 & 2.75±0.02 & 82.19±0.21&    1.51±0.09&    31.42±0.94  & 
68.76±0.43  &  3.04±0.16&    32.58±1.07
 \\
& CaT     &  65.60±0.36 &   7.11±0.42 &   62.67±1.08  & 
80.32±0.40  &  0.77±0.09 &   102.84±1.96  & 
61.99±1.41  &  5.00±11.58  & 172.40±2.39

 \\
& TACO &  78.61±0.42&   -5.17±0.51 &  43.00±3.26  &
\textbf{83.97±0.38}  & -0.43±0.33 &  140.95±12.64  &
70.73±0.45 &  -6.57±0.41 & 46.98±8.41
 \\

\midrule

\multirow{5}{*}{Parameter-isolation}
& TPP &  71.21±0.24	&  0.00±0.00	&  19.87±0.11	&  83.49±0.10&  	0.00±0.00&  	65.30±0.53	&  70.71±0.17&  	0.00±0.00&  	31.55±0.13 \\
& TAAM  &  78.09±0.28	&	0.00±0.00	&	18.88±0.31	&	79.86±0.03	&	0.00±0.00	&	30.96±0.89	&	70.21±0.40	&	-3.98±1.13	&	66.91±0.27\\
& CAGNN & 77.56±2.03&-8.23±7.88& 84.16±9.82 &78.12±0.97&-4.33±0.97&281.14±30.31 &70.35±1.22&-6.67±0.08&179.22±0.17\\
& SAFER & 78.32±1.18&0.00±0.00& 42.74±0.54 & 78.81±0.35&0.00±0.00&	105.97±0.40 & 	70.86±0.08&-3.57±0.71&	92.07±0.31\\
& \model (Ours)  &  \textbf{84.54±1.01}	&	-0.88±1.52	&	7.82±0.06	&	82.76±0.13	&	0.00±0.00	&	17.77±0.18	&	\textbf{77.18±1.81}	&	-1.19±2.07	&	17.15±0.53\\
\midrule

\multirow{1}{*}{Full dataset} 
& Joint &  85.29±2.52 & -0.33±0.09 &30.37±0.62&
88.12±1.52  &  0.94±0.15   &107.58±2.84&
77.41±1.66  &  2.14±0.65&81.55±1.97\\
\bottomrule

\end{tabular}
}

\label{tab:main_results_inductive}
\end{table*}

%% file: tables/mixed_setting_ratio1.tex
\begin{table*}[t]
\centering
\Huge
\caption{Mixed-setting results of \model and the strongest replay and replay-free baselines. Bold: best AP. On DBLP$_{mixed}$, the \model row was rerun on an NVIDIA A100 (80GB) with the same non-deterministic GPU kernels as the baselines.}
\setlength{\tabcolsep}{2.5pt}
\resizebox{0.92\textwidth}{!}{
\renewcommand{\arraystretch}{1.1}
\begin{tabular}{lccccccccccccccc}

\hline
\multicolumn{15}{c}{$\textbf{Kindle}_{mixed}$} \\
\cmidrule(lr){2-16}
Trans: Induct 
& \multicolumn{3}{c}{1:9}
& \multicolumn{3}{c}{3:7}
& \multicolumn{3}{c}{5:5}
& \multicolumn{3}{c}{7:3}
& \multicolumn{3}{c}{9:1} \\
\cmidrule(lr){2-4} \cmidrule(lr){5-7} \cmidrule(lr){8-10} \cmidrule(lr){11-13} \cmidrule(lr){14-16}
Methods
& AP (\%) $\uparrow$ & BWT (\%) $\uparrow$ & Time (secs) $\downarrow$
& AP (\%) $\uparrow$ & BWT (\%) $\uparrow$  & Time (secs) $\downarrow$
& AP (\%) $\uparrow$ & BWT (\%) $\uparrow$  & Time (secs) $\downarrow$
& AP (\%) $\uparrow$ & BWT (\%) $\uparrow$  & Time (secs) $\downarrow$
& AP (\%) $\uparrow$ & BWT (\%) $\uparrow$   & Time (secs) $\downarrow$ \\
\hline
TACO     & 56.56±0.11	& -2.84±0.05	& 42.59±0.17	& 55.95±0.14	& -3.27±0.02	& 47.18±0.09	& 61.51±0.11& 	-0.64±0.03	& 42.67±0.14	& 61.16±0.05& 	-1.37±0.02& 	43.50±0.09	& 59.18±0.10	& -3.57±0.04	& 46.42±0.13\\
TPP      &  71.37±0.03& 	0.00±0.00	& 18.90±0.08& 	71.53±0.05	& 0.00±0.00	& 18.90±0.03& 	72.22±0.04& 	0.00±0.00	& 18.70±0.02	& 69.40±0.09	& 0.00±0.00	& 18.70±0.06	& 70.78±0.05& 	0.00±0.00& 	18.80±0.03\\
TAAM      & 82.36±0.08& 	0.00±0.00& 	18.50±0.07& 	83.50±0.09& 	0.00±0.00& 	18.20±0.07& 	81.10±0.08& 	0.00±0.00	& 18.40±0.10& 	83.90±0.05	& 0.00±0.00	& 18.40±0.11& 82.91±0.06& 	0.00±0.00& 	18.50±0.05 \\
CAGNN & 82.16±1.83 & -5.43±0.62 & 82.56±10.26 & 83.92±0.75 & -1.88±1.56 & 81.89±10.06 & 82.73±2.44 & -3.44±1.43 & 82.09±9.51 & 84.73±0.64 & -4.19±0.93 & 81.43±10.39 & 85.12±2.52 & -3.96±2.81 & 82.56±10.06 \\
SAFER & 80.80±0.76 & 0.00±0.00 & 39.84±0.30 & 82.99±0.16 & 0.00±0.00 & 40.28±0.21 & 82.80±0.53 & 0.00±0.00 & 39.87±0.58 & 84.10±0.93 & 0.00±0.00 & 40.50±0.13 & 84.34±0.01 & 0.00±0.00 & 40.19±0.19 \\
\model &  \textbf{86.68±0.02}	& 0.00±0.00	& 6.71±0.03	& \textbf{84.18±0.00}& 	-4.44±0.16& 	5.84±0.02& 	\textbf{85.28±0.04}& 	0.00±0.00	& 7.01±0.00& 	\textbf{88.35±0.06}& 	0.00±0.00	& 7.07±0.03& 	\textbf{89.17±0.06}& 	0.00±0.00	& 6.82±0.09 \\
\hline

\hline
\multicolumn{15}{c}{$\textbf{DBLP}_{mixed}$} \\
\cmidrule(lr){2-16}
Trans: Induct 
& \multicolumn{3}{c}{1:9}
& \multicolumn{3}{c}{3:7}
& \multicolumn{3}{c}{5:5}
& \multicolumn{3}{c}{7:3}
& \multicolumn{3}{c}{9:1} \\
\cmidrule(lr){2-4} \cmidrule(lr){5-7} \cmidrule(lr){8-10} \cmidrule(lr){11-13} \cmidrule(lr){14-16}
Methods
& AP (\%) $\uparrow$ & BWT (\%) $\uparrow$ & Time (secs) $\downarrow$
& AP (\%) $\uparrow$ & BWT (\%) $\uparrow$ & Time (secs) $\downarrow$
& AP (\%) $\uparrow$ & BWT (\%) $\uparrow$ & Time (secs) $\downarrow$
& AP (\%) $\uparrow$ & BWT (\%) $\uparrow$ & Time (secs) $\downarrow$
& AP (\%) $\uparrow$ & BWT (\%) $\uparrow$ & Time (secs) $\downarrow$ \\
\hline
TACO     & 79.55±0.09	&	-7.75±0.11	&	79.11±0.13	&	80.90±0.15	&	-3.61±0.08	&	88.12±0.16	&	64.59±0.13	&	-20.94±0.53	&	89.97±0.13	&	66.75±0.13	&	-9.66±0.09	&	95.58±0.17	&	62.95±0.16	&	-0.45±0.07& 124.37±0.24	\\
TPP & \textbf{89.51±0.19} & 0.00±0.00 & 38.08±1.25 & \textbf{89.03±0.02} & 0.00±0.00 & 39.29±2.11 & 89.51±0.08 & 0.00±0.00 & 39.10±2.04 & 89.32±0.09 & 0.00±0.00 & 37.87±1.77 & 89.16±0.15 & 0.00±0.00 & 37.90±1.66 \\
TAAM & 87.54±0.23 & 0.00±0.00 & 33.93±1.22 & 88.33±0.57 & 0.00±0.00 & 33.67±1.55 & 88.75±0.02 & 0.00±0.00 & 34.24±1.14 & 89.04±0.21 & 0.00±0.00 & 35.18±1.28 & 89.78±0.05 & 0.00±0.00 & 35.67±1.54 \\
CAGNN & 85.05±0.13 & -1.64±0.08 & 245.18±87.10 & 85.68±2.26 & -2.29±3.61 & 283.74±38.58 & 84.08±3.85 & -4.91±4.25 & 247.14±87.61 & 79.17±0.16 & -10.32±0.19 & 248.55±87.84 & 73.61±5.36 & -16.75±6.03 & 250.12±89.00 \\
SAFER & 87.09±0.16 & 1.02±0.71 & 105.12±0.45 & 87.92±0.06 & 0.00±0.00 & 105.43±0.37 & 88.88±0.57 & 0.00±0.00 & 106.67±0.62 & 88.73±0.04 & 0.00±0.00 & 106.86±0.26 & 88.51±0.04 & 0.00±0.00 & 107.42±0.74 \\
\model & 88.04±1.06 & 0.16±1.04 & 19.70±0.99 & 86.53±2.52 & 1.20±2.42 & 19.46±0.66 & \textbf{90.10±0.21} & 0.00±0.00 & 17.40±0.81 & \textbf{90.26±0.16} & 0.00±0.00 & 19.09±1.32 & \textbf{90.34±0.32} & -0.08±0.14 & 20.06±0.61\\
\hline

\hline
\multicolumn{15}{c}{$\textbf{ACM}_{mixed}$} \\
\cmidrule(lr){2-16}
Trans: Induct 
& \multicolumn{3}{c}{1:9}
& \multicolumn{3}{c}{3:7}
& \multicolumn{3}{c}{5:5}
& \multicolumn{3}{c}{7:3}
& \multicolumn{3}{c}{9:1} \\
\cmidrule(lr){2-4} \cmidrule(lr){5-7} \cmidrule(lr){8-10} \cmidrule(lr){11-13} \cmidrule(lr){14-16}
Methods
& AP (\%) $\uparrow$ & BWT (\%) $\uparrow$ & Time (secs) $\downarrow$
& AP (\%) $\uparrow$ & BWT (\%) $\uparrow$ & Time (secs) $\downarrow$
& AP (\%) $\uparrow$ & BWT (\%) $\uparrow$ & Time (secs) $\downarrow$
& AP (\%) $\uparrow$ & BWT (\%) $\uparrow$ & Time (secs) $\downarrow$
& AP (\%) $\uparrow$ & BWT (\%) $\uparrow$ & Time (secs) $\downarrow$ \\
\hline
TACO     & 69.27±0.09	&	-11.66±0.16	&	31.27±0.25	&	71.75±0.13	&	-9.26±0.06	&	32.52±0.09	&	75.04±0.14	&	-2.50±0.09	&	30.58±0.09	&	62.43±0.16	&	-18.10±0.35	&	30.96±0.42	&	72.39±0.16	&	0.00±0.00	&	34.37±0.15\\
TPP      & 65.03±0.07	&	0.00±0.00	&	31.30±0.25	&	67.57±0.05	&	0.00±0.00	&	30.90±0.11	&	66.07±0.09	&	-4.40±0.21	&	31.10±0.13	&	76.50±0.19	&	0.00±0.00	&	31.30±0.08	&	68.87±0.04	&	0.00±0.00	&	31.13±0.12 \\
TAAM      &  76.84±0.02	&	0.00±0.00	&	57.25±0.13	&	76.96±0.08	&	-1.56±0.06	&	57.64±0.23	&	81.06±0.08	&	0.45±0.00	&	57.50±0.02	&	83.75±0.06	&	0.41±0.00	&	57.51±0.03	&	84.39±0.05	&	-1.11±0.00	&	57.44±0.25	\\
CAGNN & 78.01±1.75 & -2.19±1.48 & 166.24±7.95 & 78.78±0.54 & -2.39±0.51 & 170.68±15.20 & 74.72±0.30 & -10.05±0.39 & 161.90±0.71 & 83.36±1.20 & -1.58±0.69 & 173.73±15.49 & 76.31±0.15 & -1.47±0.23 & 172.78±12.08 \\
SAFER & 68.28±0.12 & -5.55±1.58 & 85.63±0.00 & 71.44±0.59 & -8.89±0.00 & 85.22±1.00 & 64.88±5.64 & -1.61±6.31 & 85.03±0.57 & 74.66±0.01 & -8.89±0.00 & 84.52±1.26 & 64.38±0.22 & -12.23±7.86 & 82.96±0.63 \\
\model &  \textbf{81.10±0.03}	&	0.00±0.00	&	20.30±0.11	&	\textbf{82.00±0.06}	&	0.00±0.00	&	21.07±0.07	&	\textbf{83.66±0.06}	&	0.00±0.00	&	21.21±0.13	&	\textbf{85.62±0.07}	&	0.00±0.00	&	20.85±0.09	&	\textbf{84.41±0.06}	&	0.00±0.00	&	20.81±0.07	\\
\hline

\hline
\multicolumn{15}{c}{$\textbf{CoraFull}_{mixed}$} \\
\cmidrule(lr){2-16}
Trans: Induct 
& \multicolumn{3}{c}{1:9}
& \multicolumn{3}{c}{3:7}
& \multicolumn{3}{c}{5:5}
& \multicolumn{3}{c}{7:3}
& \multicolumn{3}{c}{9:1} \\
\cmidrule(lr){2-4} \cmidrule(lr){5-7} \cmidrule(lr){8-10} \cmidrule(lr){11-13} \cmidrule(lr){14-16}
Methods
& AP (\%) $\uparrow$ & BWT (\%) $\uparrow$ & Time (secs) $\downarrow$
& AP (\%) $\uparrow$ & BWT (\%) $\uparrow$ & Time (secs) $\downarrow$
& AP (\%) $\uparrow$ & BWT (\%) $\uparrow$ & Time (secs) $\downarrow$
& AP (\%) $\uparrow$ & BWT (\%) $\uparrow$ & Time (secs) $\downarrow$
& AP (\%) $\uparrow$ & BWT (\%) $\uparrow$  & Time (secs) $\downarrow$\\
\hline
CaT     &  35.24±0.31	&	-26.17±0.43	&	1520.25±4.13	&	28.96±2.13	&	-16.97±5.01	&	1481.30±6.11	&	43.60±0.63	&	-8.16±0.17	&	1530.83±3.34	&	59.41±0.42	&	-6.89±0.09	&	1478.80±0.98	&	75.24±0.09	&	-5.18±0.10	&	1465.40±2.81\\
TPP      & 89.59±0.02	&	0.00±0.00	&	91.54±0.29	&	90.55±0.01	&	0.00±0.00	&	88.25±0.32	&	89.85±0.03	&	0.00±0.00	&	83.50±0.11	&	90.03±0.07	&	0.00±0.00	&	79.12±0.23	&	91.57±0.03	&	0.00±0.00	&	77.47±0.33	\\
TAAM      &  90.89±0.02	&	0.00±0.00	&	201.91±0.76	&	90.53±0.03	&	-0.05±0.00	&	184.25±0.47	&	89.63±0.01	&	0.00±0.00	&	175.31±0.65	&	90.23±0.02	&	0.00±0.00	&	167.49±0.31	&	89.37±0.03	&	-0.05±0.00	&	153.07±0.85	\\

CAGNN & 25.67±2.39 & -15.50±1.08 & 507.62±64.91 & 43.28±1.91 & -9.68±0.37 & 504.67±62.74 & 56.91±2.45 & -6.49±0.92 & 495.87±61.04 & 76.28±0.08 & -1.99±0.16 & 491.94±61.68 & 88.98±0.30 & -0.04±0.01 & 489.27±55.37 \\
SAFER & 53.44±0.95 & 0.00±0.00 & 257.85±12.86 & 61.91±1.29 & 0.00±0.00 & 246.55±5.49 & 71.87±1.68 & 0.00±0.00 & 244.06±4.04 & 82.48±0.13 & 0.00±0.00 & 228.95±10.73 & 90.42±1.80 & 0.00±0.00 & 226.59±6.87 \\

\model & \textbf{96.09±0.01}	&	0.00±0.00	&	66.74±0.13	&	\textbf{96.19±0.01}	&	0.00±0.00	&70.26±0.28	&	\textbf{94.28±0.02}	&	0.00±0.00	&	64.27±0.16	&	\textbf{96.17±0.02}	&	0.00±0.00	&	60.57±0.22	&	\textbf{95.97±0.01}	&	0.00±0.00	&	69.25±0.23	\\
\hline

\hline
\multicolumn{15}{c}{$\textbf{Arxiv}_{mixed}$} \\
\cmidrule(lr){2-16}
Trans: Induct 
& \multicolumn{3}{c}{1:9}
& \multicolumn{3}{c}{3:7}
& \multicolumn{3}{c}{5:5}
& \multicolumn{3}{c}{7:3}
& \multicolumn{3}{c}{9:1} \\
\cmidrule(lr){2-4} \cmidrule(lr){5-7} \cmidrule(lr){8-10} \cmidrule(lr){11-13} \cmidrule(lr){14-16}
Methods
& AP (\%) $\uparrow$ & BWT (\%) $\uparrow$ & Time (secs) $\downarrow$
& AP (\%) $\uparrow$ & BWT (\%) $\uparrow$ & Time (secs) $\downarrow$
& AP (\%) $\uparrow$ & BWT (\%) $\uparrow$ & Time (secs) $\downarrow$
& AP (\%) $\uparrow$ & BWT (\%) $\uparrow$ & Time (secs) $\downarrow$
& AP (\%) $\uparrow$ & BWT (\%) $\uparrow$ & Time (secs) $\downarrow$ \\
\hline
CaT     &  30.91±0.13	&	1.12±0.08	&	359.38±2.01	&	23.78±2.18	&	-7.20±0.09	&	313.62±0.74	&	37.14±0.14	&	-4.95±0.02	&	274.03±0.68	&	49.85±0.19	&	-6.65±0.05	&	299.25±1.25	&	61.51±0.09	&	-7.03±0.07	&	276.49±0.48\\
TPP      &  94.59±0.00	&	0.00±0.00	&	54.80±0.14	&	94.55±0.00	&	0.00±0.00	&	53.97±0.16	&	92.39±0.00	&	-2.29±0.08	&	54.13±0.14	&	94.74±0.00	&	0.00±0.00	&	52.03±0.15	&	92.41±0.00	&	-1.89±0.00	&	49.82±0.14\\
TAAM      & 81.92±0.03	&	-3.95±0.00	&	67.77±0.00	&	83.25±0.05	&	-2.21±0.00	&	64.08±0.00	&	80.16±0.06	&	-4.73±0.00	&	64.38±0.00	&	82.01±0.02	&	-2.99±0.00	&	63.34±0.00	&	83.17±0.04	&	-1.43±0.00	&	61.15±0.00\\
CAGNN & 25.14±2.96 & -2.79±3.63 & 342.78±12.83 & 40.41±2.22 & -1.07±0.82 & 316.35±0.11 & 54.31±5.28 & -1.29±1.58 & 322.37±22.92 & 70.02±2.52 & -0.05±0.10 & 298.77±4.93 & 86.57±0.38 & 0.01±0.01 & 286.79±11.31 \\
SAFER & 41.02±0.04 & -6.86±0.14 & 158.49±3.92 & 56.63±2.71 & -3.41±1.73 & 150.69±0.08 & 64.33±0.53 & -5.46±1.76 & 147.33±1.93 & 77.22±2.68 & -2.23±3.15 & 147.96±3.83 & 88.78±0.20 & 0.00±0.00 & 140.61±0.92 \\
\model &  \textbf{94.73±0.01}	&	0.00±0.00	&	33.52±0.14	&	\textbf{95.02±0.02}	&	0.00±0.00	&	37.06±0.09	&	\textbf{95.35±0.01}	&	0.00±0.00	&	34.74±0.12	&	\textbf{95.28±0.02}	&	0.00±0.00	&	36.15±0.08	&	\textbf{95.30±0.01}	&	0.00±0.00	&	20.14±0.06	\\
\hline

\hline
\multicolumn{15}{c}{$\textbf{Reddit}_{mixed}$} \\
\cmidrule(lr){2-16}
Trans: Induct 
& \multicolumn{3}{c}{1:9}
& \multicolumn{3}{c}{3:7}
& \multicolumn{3}{c}{5:5}
& \multicolumn{3}{c}{7:3}
& \multicolumn{3}{c}{9:1} \\
\cmidrule(lr){2-4} \cmidrule(lr){5-7} \cmidrule(lr){8-10} \cmidrule(lr){11-13} \cmidrule(lr){14-16}
Methods
& AP (\%) $\uparrow$ & BWT (\%) $\uparrow$ & Time (secs) $\downarrow$
& AP (\%) $\uparrow$ & BWT (\%) $\uparrow$ & Time (secs) $\downarrow$
& AP (\%) $\uparrow$ & BWT (\%) $\uparrow$ & Time (secs) $\downarrow$
& AP (\%) $\uparrow$ & BWT (\%) $\uparrow$ & Time (secs) $\downarrow$
& AP (\%) $\uparrow$ & BWT (\%) $\uparrow$ & Time (secs) $\downarrow$ \\
\hline
CaT     &  80.22±0.11	&	-12.36±0.13	&	1366.10±5.26	&	88.46±0.42	&	-0.77±0.11	&	1290.61±4.13	&	88.49±0.13	&	-1.88±0.08	&	1263.93±3.42	&	95.78±0.05	&	-0.15±0.04	&	1185.39±3.74	&	93.23±0.04	&	-0.68±0.00	&	1063.40±4.14	 \\
TPP      &  99.32±0.02	&	0.00±0.00	&	470.10±1.13	&	99.19±0.01	&	0.00±0.00	&	464.50±2.39	&	99.11±0.02	&	0.00±0.00	&	443.65±3.25	&	99.33±0.03	&	0.00±0.00	&	397.17±5.13	&	99.24±0.02	&	0.00±0.00	&	360.49±6.13 \\
TAAM      &  99.52±0.03	&	0.00±0.00	&	455.78±0.98	&	99.54±0.01	&	0.00±0.00	&	420.28±1.33	&	99.53±0.02	&	-0.03±0.00	&	394.98±0.85	&	99.62±0.02	&	0.00±0.00	&	360.88±0.64	&	99.54±0.01	&	-0.01±0.00	&	353.58±0.98	 \\
CAGNN & 53.74±0.09 & -4.23±0.05 & 1350.25±2.33 & 63.28±0.11 & -2.70±0.02 & 1312.40±2.33 & 69.87±0.05 & -7.10±0.01 & 1244.07±2.98 & 84.81±0.08 & 0.02±0.00 & 1147.44±5.38 & 95.73±0.04 & -0.02±0.00 & 1520.40±8.87 \\
SAFER & 55.01±0.08 & 0.00±0.00 & 562.40±5.54 & 64.78±0.09 & 0.00±0.00 & 500.52±6.33 & 74.79±0.07 & 0.00±0.00 & 549.50±7.53 & 84.59±0.09 & 0.00±0.00 & 501.23±3.77 & 94.44±0.09 & 0.00±0.00 & 450.25±0.33 \\
\model & \textbf{99.64±0.02}	&	0.00±0.00	&	220.10±2.13	&	\textbf{99.67±0.03}	&	0.00±0.00	&	207.72±1.03	&	\textbf{99.54±0.02}	&	0.00±0.00	&	205.74±0.24	&	\textbf{99.69±0.01}	&	0.00±0.00	&	190.65±0.43	&	\textbf{99.55±0.04}	&	0.00±0.00	&	163.76±0.22	 \\
\hline

\hline
\multicolumn{15}{c}{$\textbf{Products}_{mixed}$} \\
\cmidrule(lr){2-16}
Trans: Induct 
& \multicolumn{3}{c}{1:9}
& \multicolumn{3}{c}{3:7}
& \multicolumn{3}{c}{5:5}
& \multicolumn{3}{c}{7:3}
& \multicolumn{3}{c}{9:1} \\
\cmidrule(lr){2-4} \cmidrule(lr){5-7} \cmidrule(lr){8-10} \cmidrule(lr){11-13} \cmidrule(lr){14-16}
Methods
& AP (\%) $\uparrow$ & BWT (\%) $\uparrow$ & Time (secs) $\downarrow$
& AP (\%) $\uparrow$ & BWT (\%) $\uparrow$ & Time (secs) $\downarrow$
& AP (\%) $\uparrow$ & BWT (\%) $\uparrow$ & Time (secs) $\downarrow$
& AP (\%) $\uparrow$ & BWT (\%) $\uparrow$ & Time (secs) $\downarrow$
& AP (\%) $\uparrow$ & BWT (\%) $\uparrow$  & Time (secs) $\downarrow$\\
\hline
CaT     & 18.67±0.43	&	-8.47±0.07	&	2260.04±2.34	&	22.09±0.41	&	-9.28±0.08	&	2214.08±0.86	&	37.95±0.34	&	-6.83±0.07	&	2157.70±6.37	&	48.89±0.08	&	-7.76±0.09	&	1884.42±3.85	&	54.70±0.50	&	-12.60±0.26	&	1681.71±4.77  \\
TPP      &  95.43±0.00	& 	0.00±0.00	& 	1160.70±3.24	& 	87.67±0.13	& 	0.00±0.00	& 	1157.59±7.31	& 91.28±0.00	&	0.00±0.00	&	1093.32±4.24	&	95.71±0.00	&	0.00±0.00	&	1067.47±3.27	&95.68±0.00	&	0.00±0.00	&	957.15±6.82\\
TAAM      &  88.12±0.09	&	-3.29±0.41	&	543.15±0.90	&	85.77±0.22	&	-4.90±0.19	&	539.36±0.74	&	85.94±0.28	&	-4.27±0.06	&	488.85±1.37	&	87.56±0.13	&	-3.33±0.08	&	456.87±0.84	&	85.44±0.40	&	-6.00±0.00	&	404.02±0.39	\\
CAGNN & 23.77±0.09 & -16.04±0.24 & 2935.54±4.22 & 34.86±0.14 & -14.35±0.11 & 2903.24±5.68 & 46.89±0.09 & -12.65±0.21 & 2520.72±5.09 & 71.28±0.08 & -3.91±0.03 & 2378.57±6.30 & 89.57±0.02 & -2.17±0.05 & 2317.29±5.79 \\
SAFER & 49.69±0.07 & -2.14±0.04 & 1229.38±4.27 & 62.36±0.06 & 0.00±0.00 & 1171.61±6.23 & 73.38±0.09 & 0.00±0.00 & 1121.23±5.64 & 84.21±0.07 & 0.00±0.00 & 1067.65±7.61 & 91.31±0.07 & 0.00±0.00 & 1094.03±14.92 \\
\model & \textbf{95.45±0.02}	&	0.00±0.00	&	653.70±0.84	&	\textbf{97.00±0.02}	&	4.46±0.00	&	523.80±0.63	&	\textbf{93.09±0.01}	&	3.82±0.05	&	390.12±0.52	&	\textbf{96.50±0.02}	&	0.00±0.00	&	449.65±0.31	&	\textbf{96.25±0.01}	&	0.00±0.00	&	402.40±0.92	\\
\hline

\end{tabular}
}
\label{tab:results_mixed}
\end{table*}

%% file: tables/ablation.tex
\begin{table}[!ht]
\centering

\caption{Ablation study of \model's four components on CoraFull (transductive) and Kindle (inductive). All rows use seeds 1024--1026, so the Final row differs slightly from Tables~\ref{tab:main_results_transductive} and~\ref{tab:main_results_inductive}. On CoraFull, a misrouted test graph scores 0.}

\resizebox{0.99\columnwidth}{!}{
\begin{tabular}{lcccccc}
\toprule
\multirow{2}{*}{\textbf{Methods}} &
\multicolumn{3}{c}{\textbf{$\text{CoraFull}$}} &
\multicolumn{3}{c}{\textbf{$\text{Kindle}$}}\\
\cmidrule(lr){2-4}\cmidrule(lr){5-7}
 & AP (\%) $\uparrow$ & BWT (\%) $\uparrow$ & Time (s) $\downarrow$
 & AP (\%) $\uparrow$ & BWT (\%) $\uparrow$ & Time (s) $\downarrow$\\
\midrule
\makecell[l]{w/o Laplacian\\ Smoothing}  & 92.93±0.62 &	0.00±0.00 &	56.01±3.13 &	44.41±0.32 &	0.00±0.00 &	8.85±1.16 \\ \midrule
\makecell[l]{w/o Bottleneck}   & 95.65±0.05	 &0.00±0.00	 &36.04±2.90 &	78.48±0.51	 &$-$0.38±0.66	 &4.84±0.69\\ \midrule
\makecell[l]{w/o Community-aware}   & 90.54±0.14	&5.59±0.05	& 54.85±1.76 &	83.09±0.07	&$-$2.89±0.03	&7.48±0.27 \\ \midrule
\makecell[l]{w/o Majority Voting}   & 89.56±1.53	 &0.00±0.00 &	54.98±1.12 &	76.31±0.78	&$-$9.20±0.81	&8.49±0.42  \\ \midrule
\makecell[l]{\model (Final)} & 95.80±0.26  &	0.00±0.00 &	59.60±1.65  &	84.52±1.09 &	$-$0.84±1.45  &	8.30±0.44 \\ \bottomrule
\end{tabular}}
\label{ablation}
\end{table}

%% file: tables/equal_memory.tex
\begin{table}[t]
\centering
\caption{Comparison at matched final storage. Italic: CAGNN in its default configuration, far above the matched storage. CaT (task-IL) is given the true task identifier at test time and is not ranked; its training, storage, and time are those of CaT.}
\label{tab:equal_memory}
\resizebox{0.99\columnwidth}{!}{
\begin{tabular}{llcccc}
\toprule
\textbf{Dataset} & \textbf{Method} & \textbf{Memory (MB)} & \textbf{AP (\%)} $\uparrow$ & \textbf{BWT (\%)} $\uparrow$ & \textbf{Time (s)} $\downarrow$ \\
\midrule
\multirow{9}{*}{\makecell[l]{Arxiv\\($T{=}20$, $d{=}128$)}}
& ER-GNN (budget 40) & 0.4288 & 48.12±1.02 & -16.05±0.36 & 43.24±2.56 \\
& SSM (budget 40)    & 0.4162 & 52.87±0.68 & -16.49±1.11 & 45.38±2.42 \\
& CaT (budget 40)    & 0.4288 & 61.69±2.58 & -12.01±1.74 & 225.76±3.23 \\
& CaT (task-IL) & 0.4288 & 94.28±0.83 & -0.02±0.25 & 225.76±3.23 \\
& TAAM (rank 8)      & 0.3863  & 84.94±0.46 & -1.43±0.43  & 71.51±0.68 \\
& SAFER ($r{=}14$)   & 0.4069  & 93.28±0.17 & \textbf{0.00±0.00} & 151.19±1.60 \\
& TPP ($p{=}17$)     & 0.3834  & 94.57±0.07 & \textbf{0.00±0.00} & 43.66±1.12 \\
& \model ($b{=}13$)  & 0.4006  & \textbf{95.17±0.07} & \textbf{0.00±0.00} & \textbf{38.09±0.76} \\
\cmidrule(lr){2-6}
& \textit{CAGNN (default)} & \textit{825.66} & \textit{93.33±0.11} & \textit{0.00±0.00} & \textit{208.57±9.59} \\
\midrule
\multirow{9}{*}{\makecell[l]{CoraFull\\($T{=}35$, $d{=}8710$)}}
& ER-GNN (budget 40) & 48.3215 & 57.74±0.67 & -8.92±0.47 & 80.46±3.10 \\
& SSM (budget 40)    & 48.3057 & 65.78±0.85 & -7.10±0.27 & 80.78±2.62 \\
& CaT (budget 40)    & 48.8096 & 79.90±0.46 & -4.97±0.65 & 951.63±81.90 \\
& CaT (task-IL) & 48.8096 & 95.82±0.14 & -0.33±0.06 & 951.63±81.90 \\
& TAAM (rank 17)     & 48.1020 & 89.45±0.15 & -0.05±0.05  & 138.99±2.73 \\
& SAFER ($p{=}20$)   & 48.0664 & 95.29±0.16 & \textbf{0.00±0.00} & 259.06±5.45 \\
& TPP ($p{=}19$)     & 46.3072 & 89.46±0.55 & \textbf{0.00±0.00} & 74.68±2.20 \\
& \model ($b{=}13$)  & 47.5587 & \textbf{95.72±0.02} & \textbf{0.00±0.00} & \textbf{67.81±1.18} \\
\cmidrule(lr){2-6}
& \textit{CAGNN (default)} & \textit{3004.06} & \textit{94.66±0.07} & \textit{0.00±0.00} & \textit{372.10±16.19} \\
\bottomrule
\end{tabular}}
\end{table}

%% file: tables/routing_correlation.tex
\begin{table}[t]
\centering
\caption{Routing accuracy of \model on the four transductive benchmarks: graph- and node-level routing accuracy (\%) and edge-error correlation (Pearson correlation of node routing errors between edge endpoints). Pooled covers all 98 test graphs over three runs; its correlation is the mean over the four benchmarks.}
\label{tab:routing_corr}
\resizebox{0.75\columnwidth}{!}{
\begin{tabular}{lccc}
\toprule
\textbf{Dataset} & \makecell{\textbf{Graph-level}\\\textbf{routing (\%)}} & \makecell{\textbf{Node-level}\\\textbf{routing (\%)}} & \makecell{\textbf{Edge-error}\\\textbf{correlation}} \\
\midrule
CoraFull & 100.00 & 75.35 & 0.899 \\
Arxiv    & 100.00 & 55.46 & 0.354 \\
Reddit   & 100.00 & 97.10 & 0.253 \\
Products & 100.00 & 60.55 & 0.820 \\
\midrule
Pooled   & 100.00 & 72.26 & 0.581 \\
\bottomrule
\end{tabular}}
\end{table}

%% file: tables/mixed_node_routing.tex
\begin{table}[t]
\centering
\caption{Mixed-task test graphs on Arxiv and Reddit ($T{=}20$): AP$\uparrow$ on single-task (Native) and mixed-task (Mixed) test graphs, and running time (s). Seeds 1042--1044 differ from those of Table~\ref{tab:main_results_transductive}. Replay buffers hold 29 (Arxiv) and 40 (Reddit) nodes per task. TWP uses two GAT heads on Reddit (eight run out of memory). \model's Native AP is from a run without mixed evaluation, with its default graph-level vote.}
\label{tab:mixed_node}
\resizebox{0.99\columnwidth}{!}{
\begin{tabular}{lcccccc}
\toprule
\multirow{2}{*}{\textbf{Method}} & \multicolumn{3}{c}{\textbf{Arxiv}} & \multicolumn{3}{c}{\textbf{Reddit}} \\
\cmidrule(lr){2-4}\cmidrule(lr){5-7}
 & Native AP & Mixed AP & Time (s) & Native AP & Mixed AP & Time (s) \\
\midrule
Finetuning & 4.92±0.00  & 1.92±0.00  & 69.50±1.08  & 4.96±0.01  & 3.50±0.01  & 532.94±1.63 \\
EWC        & 5.00±0.11  & 1.96±0.08  & 149.35±1.20 & 6.42±1.21  & 4.73±1.14  & 608.91±2.15 \\
MAS        & 5.13±0.12  & 2.44±0.13  & 114.69±29.69 & 8.61±0.26  & 7.31±0.42  & 569.80±2.80 \\
TWP        & 5.65±1.83  & 8.21±2.61  & 318.71±1.28 & 2.05±3.26        & 1.79±1.45         & 3247.51±252.87     \\
ER-GNN     & 27.65±0.24 & 28.79±0.58 & 67.72±1.26  & 30.13±4.59 & 27.52±5.80 & 539.59±10.67 \\
SSM        & 32.27±1.01 & 33.52±1.06 & 68.48±0.97  & 61.93±2.75 & 55.80±1.10 & 550.62±28.38 \\
CaT        & 40.27±0.31 & 40.12±0.07 & 310.90±26.38 & 62.13±4.13 & 53.76±4.23 & 1534.39±26.44 \\
TPP        & 94.42±0.14 & 10.00±0.03 & 49.63±1.52 & 99.29±0.04 & 72.87±0.00 & 461.47±5.08 \\
SAFER      & 93.28±0.10 & 28.16±0.26 & 137.57±2.47 & 99.56±0.02 & 59.32±1.71 & 759.34±3.29 \\
CAGNN      & 93.39±0.19 & 11.06±0.03 & 261.73±2.61 & \textbf{99.62±0.02} & 70.47±0.01 & 2093.53±319.58 \\
TAAM       & 83.29±1.55 & 37.54±0.11 & 61.33±2.65 & 99.58±0.02 & 50.70±0.01 & 464.84±1.84 \\
\model     & \textbf{94.90±0.12} & \textbf{55.86±0.35} & \textbf{31.74±1.33} & 99.58±0.01 & \textbf{96.59±0.06} & \textbf{244.89±2.03} \\
\midrule
Joint (reference) & 60.47±0.17 & 69.33±0.11 & 276.08±0.99 & 98.03±0.00 & 98.32±0.00 & 4753.71±20.12 \\
\bottomrule
\end{tabular}}
\end{table}

%% file: tables/noisy_edges.tex
\begin{table}[t]
\centering
\caption{Robustness to random edge noise: AP$\uparrow$ after adding random undirected edges numbering $x\%$ of each task's edge list. Single run (seed 42, 200 epochs), so the 0\% column differs from Tables~\ref{tab:main_results_transductive} and~\ref{tab:main_results_inductive} (by up to 2.61 points for \model).}
\label{tab:noisy}
\resizebox{0.8\columnwidth}{!}{
\begin{tabular}{llccccc}
\toprule
\textbf{Dataset} & \textbf{Method} & \textbf{0\%} & \textbf{+10\%} & \textbf{+20\%} & \textbf{+30\%} & \textbf{+50\%} \\
\midrule
\multirow{5}{*}{CoraFull} & TPP & 90.57 & 87.99 & 85.90 & 83.61 & 80.20 \\
                          & TAAM & 89.65 & 87.61 & 85.78 & 85.92 & 83.00 \\
                          & SAFER & 94.86 & 90.49 & 86.94 & 86.35 & 82.54 \\
                          & CAGNN & 94.48 & 87.56 & 85.92 & 85.38 & 83.37 \\
                          & \model & \textbf{95.43} & \textbf{90.53} & \textbf{88.03} & \textbf{88.03} & \textbf{85.25} \\
\midrule
\multirow{5}{*}{Arxiv} & TPP & 94.58 & 94.20 & 93.28 & 92.35 & 90.93 \\
                          & TAAM & 82.28 & 84.86 & 80.89 & 81.31 & 85.00 \\
                          & SAFER & 93.34 & 93.79 & 92.62 & 91.80 & 90.10 \\
                          & CAGNN & 92.72 & 92.31 & 91.04 & 89.78 & 88.64 \\
                          & \model & \textbf{94.86} & \textbf{95.17} & \textbf{94.34} & \textbf{93.57} & \textbf{92.59} \\
\midrule
\multirow{5}{*}{ACM}      & TPP  & 70.71 & 68.49 & 67.83 & 66.69 & 64.14 \\
                          & TAAM & 71.54 & 68.36 & 66.64 & 65.91 & 60.78 \\
                          & SAFER & 74.40 & 71.69 & 66.70 & 68.33 & 66.65 \\
                          & CAGNN & 71.89 & 69.93 & 69.91 & 70.95 & 66.70 \\
                          & \model & \textbf{74.57} & \textbf{76.48} & \textbf{72.76} & \textbf{72.93} & \textbf{70.27} \\
\midrule
\multirow{5}{*}{DBLP} & TPP & \textbf{83.40} & 79.31 & 77.29 & 75.74 & 73.11 \\
                          & TAAM & 81.61 & 79.03 & 78.34 & 77.07 & 76.89 \\
                          & SAFER & 81.00 & 78.19 & 77.24 & 76.85 & 75.18 \\
                          & CAGNN & 77.69 & 74.09 & 74.90 & 71.08 & 64.79 \\
                          & \model & 82.45 & \textbf{80.74} & \textbf{79.91} & \textbf{79.04} & \textbf{78.83} \\
\bottomrule
\end{tabular}}
\end{table}

%% file: tables/modularity_q.tex
\begin{table}[t]
\centering
\caption{Pearson $r$ and Spearman $\rho$ (two-sided $p$) between per-task AP and modularity $Q$ for \model. Baselines: number of baselines (TPP, TAAM, SAFER, CAGNN) whose AP--$Q$ correlation reaches $p<0.05$ (Pearson / Spearman).}
\label{tab:modularity}
\resizebox{0.95\columnwidth}{!}{
\begin{tabular}{lcccccc}
\toprule
\textbf{Dataset} & \textbf{Tasks} & \textbf{$Q$ range} & \textbf{$r$ ($p$)} & \textbf{$\rho$ ($p$)} & \textbf{Baselines} \\
\midrule
CoraFull & 35 & 0.04--0.50 & 0.24 (0.16) & 0.28 (0.11) & 0/4 \,/\, 0/4 \\
Arxiv    & 20 & 0.00--0.50 & $-$0.12 (0.60) & $-$0.38 (0.10) & 0/4 \,/\, 0/4 \\
Reddit   & 20 & 0.04--0.50 & 0.14 (0.55) & 0.26 (0.26) & 0/3 \,/\, 0/3 \\
Products & 23 & 0.00--0.50 & 0.11 (0.61) & 0.49 (0.02) & 0/3 \,/\, 3/3 \\
\midrule
Kindle   & 5  & 0.40--0.66 & $-$0.85 (0.07) & $-$0.60 (0.28) & 3/4 \,/\, 1/4 \\
DBLP     & 10 & 0.31--0.48 & 0.62 (0.06) & 0.85 ($<$0.01) & 1/4 \,/\, 3/4 \\
ACM      & 10 & 0.05--0.47 & $-$0.17 (0.64) & 0.07 (0.85) & 0/4 \,/\, 0/4 \\
\bottomrule
\end{tabular}}
\end{table}

%% file: tables/runtime_breakdown.tex
\begin{table*}[t]
\centering
\caption{Running-time breakdown of \model: wall-clock time of each component in seconds, with its share of the total in parentheses (NVIDIA RTX A4000). \textit{Re-eval.} evaluates all earlier tasks after each new task and is needed only for AP and BWT.}
\label{tab:runtime}
\resizebox{0.88\textwidth}{!}{
\begin{tabular}{llc|ccccccccc|c}
\toprule
\textbf{Dataset} & \textbf{Setting} & \textbf{Tasks} & \makecell{\textbf{Data}\\\textbf{transfer}} & \makecell{\textbf{Laplacian}\\\textbf{smoothing}} & \makecell{\textbf{Community}\\\textbf{detection}} & \makecell{\textbf{Prototype}\\\textbf{extraction}} & \makecell{\textbf{Adapter}\\\textbf{training}} & \textbf{Routing} & \textbf{Classification} & \makecell{\textbf{Re-eval.}\\\textbf{(protocol)}} & \textbf{Other} & \textbf{Total} \\
\midrule
CoraFull & Trans. & 35 & 0.41 (0.6\%) & 0.21 (0.3\%) & 0.18 (0.3\%) & 0.06 (0.1\%) & 52.46 (82.9\%) & 0.27 (0.4\%) & 0.05 (0.1\%) & 8.80 (13.9\%) & 0.85 (1.3\%) & 63.30 \\
Arxiv & Trans. & 20 & 0.09 (0.3\%) & 0.12 (0.4\%) & 2.69 (8.9\%) & 0.12 (0.4\%) & 25.44 (83.8\%) & 0.07 (0.2\%) & 0.02 (0.1\%) & 1.24 (4.1\%) & 0.57 (1.9\%) & 30.35 \\
Reddit & Trans. & 20 & 2.31 (1.2\%) & 1.64 (0.9\%) & 133.99 (70.1\%) & 0.12 (0.1\%) & 34.03 (17.8\%) & 0.09 (0.0\%) & 0.04 (0.0\%) & 18.13 (9.5\%) & 0.85 (0.4\%) & 191.20 \\
Products & Trans. & 23 & 3.72 (0.9\%) & 1.91 (0.5\%) & 294.64 (69.6\%) & 1.57 (0.4\%) & 66.09 (15.6\%) & 0.10 (0.0\%) & 0.06 (0.0\%) & 54.60 (12.9\%) & 0.81 (0.2\%) & 423.51 \\
Kindle & Ind. & 5 & 0.11 (1.1\%) & 0.12 (1.2\%) & 0.68 (6.9\%) & 0.02 (0.2\%) & 8.35 (84.3\%) & 0.01 (0.1\%) & 0.01 (0.1\%) & 0.13 (1.3\%) & 0.48 (4.9\%) & 9.91 \\
DBLP & Ind. & 10 & 0.13 (0.6\%) & 0.12 (0.6\%) & 4.55 (21.9\%) & 0.05 (0.2\%) & 14.95 (71.9\%) & 0.03 (0.1\%) & 0.01 (0.1\%) & 0.43 (2.1\%) & 0.51 (2.5\%) & 20.78 \\
ACM & Ind. & 10 & 0.53 (1.8\%) & 0.18 (0.6\%) & 1.42 (4.8\%) & 0.05 (0.2\%) & 25.78 (87.4\%) & 0.03 (0.1\%) & 0.03 (0.1\%) & 0.95 (3.2\%) & 0.52 (1.8\%) & 29.48 \\
\bottomrule
\end{tabular}}
\end{table*}

%% file: section/ch5_related_works.tex
\section{Related Work}

\noindent\textbf{Graph data systems and model maintenance.} Graph processing systems manage large, evolving graphs~\cite{sakr2021future}, and graph streaming frameworks apply edge insertions and deletions to the stored graph as they arrive~\cite{besta2023practice}. Systems for learning on such graphs train temporal GNNs on billion-scale graphs~\cite{zhou2022tgl}, train dynamic GNNs over sliding windows of a graph stream~\cite{chen2023neutronstream}, update node representations incrementally in a distributed dataflow~\cite{guliyev2024d3gnn}, decouple graph propagation from prediction on large dynamic graphs~\cite{zheng2023decoupled}, or compress the graph so that GNN inference runs on a smaller one~\cite{fan2025ifgc}. These systems keep the evolving graph, or a compressed copy of it, in storage and predict over a single graph rather than over classes that arrive task by task. Data management research also mitigates the impact of data drift on machine learning pipelines~\cite{dong2024datadrift} and minimizes the data collected for machine learning over Internet of Things (IoT) data streams~\cite{shaowang2025dataminimization}. \model addresses the complementary problem of maintaining a model learned from an evolving graph without retaining \pd{its own} copy of past graph data.

\noindent\textbf{Continual graph learning.} 
Classical continual learning methods fall into three families: regularization (EWC~\cite{Kirkpatrick2017ewc}, MAS~\cite{Aljundi2018mas}), memory replay (GEM~\cite{lopez2017gem}), and parameter isolation (Progressive Nets~\cite{rusu2016progressive}). These have been extended to graphs through topology-aware regularization (TWP~\cite{liu2021twp}); replay of exemplar, sparsified, or condensed subgraphs (ER-GNN~\cite{zhou2021ergnn}, SSM~\cite{zhang2022ssm}, CaT~\cite{liu2023cat}, PUMA~\cite{liu2025puma}); topology-aware embedding memory (PDGNN~\cite{zhang2024pdgnn}); and topology-aware coarsening (TACO~\cite{han2024taco}). Recent replay-free frameworks use spectral profiles plus graph prompts (TPP~\cite{niu2024tpp}), neural synapse modulators (TAAM~\cite{liu2026taam}), hierarchical parameter masking (CAGNN~\cite{xiao2026cagnn}), or subspace-aware feature reshaping (SAFER~\cite{zhang2026safer}). Unlike replay-based methods, which store raw, condensed, or coarsened graphs and replay them in later training (e.g., CaT and TACO), \model spends a constant amount of storage per task on an adapter and at most $M$ community-level prototypes; the prototypes are fixed statistics used only to route test graphs, never to retrain. Unlike existing replay-free methods, it avoids any pre-trained or per-task GNN backbone, yielding substantially lower training cost (Secs.~\ref{subsec:complexity} and~\ref{main-indivisual}). \model precomputes the $K$-hop smoothing outside the epoch loop (Sec.~\ref{sec:sgc_extraction}), whereas TPP optimizes graph prompts within it, summarizes each task by a single mean, and stores prototypes per task; TACO maintains one coarsened graph of all previous tasks. TAAM keeps a frozen SGC but runs node-level modulators in every forward pass, and PDGNN and PUMA precompute propagation, as \model does, but keep a replay buffer. Appendix~D.1 compares \model with TPP and TAAM.

\noindent\textbf{Community detection.} \model requires only a disjoint partition of nodes; we default to Louvain~\cite{blondel2008fast} and verify that Leiden~\cite{traag2019from}, Infomap~\cite{rosvall2008maps}, Fast-Greedy~\cite{clauset2004finding}, and Label Propagation~\cite{raghavan2007near} all preserve \model's superiority (Sec.~\ref{diverse-community-detection}). Parallel community detection and graph clustering frameworks scale such partitioning to large graphs~\cite{shi2021scalable,yu2024pcbs}, and \model can use them without change. Community structure has also been used in static GRL via Community-as-a-Node abstractions~\cite{park2023two,park2024tkde} and multi-level GRL~\cite{lim2025mlgrl}; unlike these reusable hierarchical representations, \model partitions each task graph and distills the result into at most $M$ prototypes per task for replay-free task identification.

%% file: section/ch6_conclusion.tex
\section{Conclusion}
\label{sec:conclusion}

We introduced \model, a replay-free CGL framework that unifies task learning, memory, and inference around a single parameter-free structural prior computed once per task via SGC. On top of this shared prior, \model trains a lightweight task-specific bottleneck adapter, extracts at most $M$ non-learned community-level prototypes per task, and, at inference time, identifies the active task by graph-level majority voting over the prototype memory. This design avoids the memory cost of replay and the recursive message passing of GNN-based methods, and is accurate and efficient across transductive, inductive, and mixed task streams. \model maintains a model over an evolving graph with constant storage per task and without keeping its own copy of past graph data.

\vspace*{-0.1cm}
\smallskip\noindent\textbf{Limitations and future work.} The default protocol assumes one task per test graph; the node-level variant (Sec.~\ref{sec:eval_routing}) still takes the graph boundary from the stream and depends on per-node routing, under which Mixed AP on Arxiv drops to 55.86\%. The design assumes homophilic graphs with community structure; Sec.~\ref{sec:eval_robust} shows robustness within this scope but does not test heterophilic graphs. Storage and matching costs grow linearly with the number of tasks (Eq.~\eqref{eq:memory}), and on the largest graphs, CPU-based Louvain takes about 70\% of the running time (Table~\ref{tab:runtime}). Future work will extend \model to heterogeneous graphs, to fully streaming regimes without explicit task boundaries, and to few-shot or semi-supervised signals for reducing the labeled data required per new task.

\section*{Use of Generative AI}
Generative AI tools were used to polish the writing and to help write scripts that check the reported numbers against the experiment logs; the method, experiments, and results are the authors' own.

%% file: tables/topm_bound.tex
\begin{table}[h]
\centering
\caption{Top-$M$ truncation at $M=10$: measured gap $\hat{\epsilon}_M$, its upper bound $2\,\mathbb{E}_{v \in V_{test}}[r(v)]$ from Corollary~1, and $\hat{\epsilon}_M$ relative to the mean pairwise squared distance between community prototypes.}
\label{tab:topm_bound}
\begin{tabular}{lccc}
\toprule
Dataset & $\hat{\epsilon}_M$ & Bound & Relative gap \\
\midrule
CoraFull & 0.1745 & 0.4906 & 9.80\% \\
Arxiv & 0.0191 & 0.2640 & 4.87\% \\
Reddit & 0.0103 & 0.0718 & 0.57\% \\
Products & 0.3300 & 1.2292 & 17.32\% \\
\bottomrule
\end{tabular}
\end{table}